%% file: DynScreen_Long_singlecol.tex
\documentclass[12pt,journal,final,onecolumn]{IEEEtran}

\input{header}

\usepackage{booktabs}

\author{Antoine Bonnefoy, 
 Valentin Emiya, 
 Liva Ralaivola, 
 R\'emi Gribonval.
 \thanks{This work was supported by Agence Nationale 
 de la Recherche (ANR), project GRETA 12-BS02-004-01. R.G. acknowledges funding 
 by the European Research Council within the PLEASE project under grant ERC-StG-2011-277906.}}
\title{Dynamic Screening: Accelerating First-Order Algorithms for the Lasso and Group-Lasso}
\begin{document}
\maketitle


\begin{abstract}
Recent computational strategies based on screening tests have been proposed to accelerate
algorithms addressing penalized sparse regression problems such as the Lasso.
Such approaches build upon the idea that it is worth dedicating some small 
computational effort to locate inactive atoms and remove them from the dictionary 
in a preprocessing stage so that the regression algorithm working with a smaller 
dictionary will then converge faster to the solution of the initial problem.
We believe that there is an even more efficient way to screen the 
dictionary and obtain a greater acceleration: inside each iteration of the regression 
algorithm, one may take advantage of the algorithm computations to obtain a 
new screening test for free with increasing screening effects along the iterations. 
The dictionary is henceforth dynamically screened instead of being screened 
statically, once and for all, before the first iteration. We formalize this dynamic 
screening principle in a general algorithmic scheme and apply it by embedding inside
a number of first-order algorithms adapted existing screening tests to solve the \lasso 
or new screening tests to solve the \glasso. 
Computational gains are assessed in a large set of 
experiments on synthetic data as well as real-world sounds and images. They 
show both the screening efficiency and the gain in terms running times.
\end{abstract}


\input{Introduction_single}

\input{basicAlgo}
\input{General_dynScreen}
\input{Experiments_single}

\input{discussion}

\bibliographystyle{IEEEtran}
\bibliography{IEEEabrv.bib,DynScreen.bib}
%

\appendices
\input{Appendix_single}

\end{document}

%% file: header.tex
\usepackage[utf8]{inputenc}
\usepackage[english]{babel}
\usepackage[T1]{fontenc}
\usepackage{graphicx}
\usepackage{amsmath,amssymb,amsthm,amsfonts,mathrsfs, dsfont}
\usepackage{ stmaryrd }
\usepackage{epstopdf}
\usepackage[usenames,dvipsnames,svgnames,table]{xcolor}
\usepackage{xspace}	
\usepackage{algorithm}
\usepackage{algorithmicx} 
\usepackage{caption}
\usepackage{algpseudocode}
\usepackage{multirow}
\usepackage{hyperref}

\graphicspath{
{../figures/}%
}

\theoremstyle{plain}
\newtheorem{thm}{Theorem}
\newtheorem{lem}[thm]{Lemma}

\theoremstyle{definition}
\newtheorem{defn}{Definition}

\newcommand{\myvec}[1]{\boldsymbol{\mathrm{#1}}}
\newcommand{\mymat}[1]{\mathbf{#1}}
\newcommand{\eqdef}{\triangleq}

\newcommand\D{\mymat{D}}
\newcommand\Id{\mymat{Id}}
\newcommand\nAt{K}
\newcommand\dimS{N}
\newcommand{\fullset}{\Gamma}
\newcommand\G{\mathcal{G}} 

\newcommand{\gf}[1]{g(#1)}
\newcommand{\gr}[1]{\left [ #1 \right ]}
\newcommand{\grindex}[2]{#1_{\gr{#2}}}

\newcommand\grwght{w}

\newcommand\atom{\myvec{d}}
\newcommand\obs{\myvec{y}}
\newcommand\pvar{\myvec{x}}
\newcommand\multipvar{\mathbf{X}}
\newcommand{\pvaru}{\myvec{u}}
\newcommand\pvary{\myvec{z}}
\newcommand\opt[1]{\tilde{#1}}
\newcommand\pvaropt{\opt{\pvar}}
\newcommand\dvar{\myvec{\theta}}

\newcommand\dvaropt{\opt{\dvar}}
\newcommand{\nvec}{\myvec n}
\newcommand{\feas}{\myvec{v}}
\newcommand{\auxvar}{\boldsymbol{\alpha}}
\newcommand\pen{\lambda}

\newcommand\R{\mathbb{R}} 

\newcommand\step{p}
\newcommand{\scr}{\mathcal{I}}
\newcommand{\comp}[1]{{#1}^c}
\newcommand{\card}[1]{\left |#1\right |}
\newcommand{\ST}{T}

\newcommand{\ie}{\textit{i.e.}\xspace} 
\newcommand{\eg}{\textit{e.g.}\xspace} 
\newcommand\lasso{\textit{Lasso}\xspace}
\newcommand\glasso{\textit{Group-Lasso}\xspace}
\newcommand\lars{\textit{LARS}\xspace}
  
\newcommand{\spreg}{\Omega}
\newcommand{\norm}[1]{\left\|#1\right\|} 
\newcommand{\normO}[1]{\left\|#1\right\|_0} 
\newcommand{\normI}[1]{\left\|#1\right\|_1} 
\newcommand{\normII}[1]{\left\|#1\right\|_2} 
\newcommand{\normInf}[1]{\left\|#1\right\|_\infty} 
\newcommand{\normM}[1]{\left\|#1\right\|} 
\newcommand{\abs}[1]{ |#1|} 
\newcommand{\argmax}[1]{\operatorname*{arg\,max}_{#1}} 
\newcommand{\argmin}[1]{\operatorname*{arg\,min}_{#1}} 

\newcommand{\sign}{\operatorname*{sign}} 
\newcommand{\coef}[2]{\mathnormal{#1}(#2)}

\newcommand{\prox}[3]{\operatorname{prox}^{#1}_{#2}\left( #3 \right)}
\newcommand{\ellips}{\mathcal{E}}

\newcommand\sph{\mathcal{S}}
\newcommand\sphrad{r}
\newcommand{\sphradf}[1]{\sphrad_{#1}}
\newcommand\sphcent{\myvec{c}}

\newcommand\recMat{\mymat{T}}
\newcommand{\safe}[1]{\mathrm{SAFE}_{#1}}
\newcommand{\gsafe}[1]{\mathrm{GSAFE}_{#1}}
\newcommand{\gsphere}{\mathrm{GSPHERE}}
\newcommand{\hsst}[1]{\mathrm{DST3}_{#1}}
\newcommand{\ghsst}[1]{\mathrm{DGST3}_{#1}}
\newcommand{\hdst}[1]{\mathrm{DDome}_{#1}}
\newcommand{\flops}[1]{\mathrm{flops_{#1}}}

\newcommand{\tx}[1]{\text{#1}}
\newcommand{\feasset}{\mathcal{V}_*}

\newcolumntype{L}[1]{>{\raggedright\let\newline\\\arraybackslash\hspace{0pt}}m{#1}}
\newcolumntype{C}[1]{>{\centering\let\newline\\\arraybackslash\hspace{0pt}}m{#1}}
\newcolumntype{R}[1]{>{\raggedleft\let\newline\\\arraybackslash\hspace{0pt}}m{#1}}

\newcommand{\dict}{\D}
\renewcommand{\it}{{t}}
\newcommand{\nAtoms}{\nAt}
\newcommand{\dimSig}{\dimS}

%% file: Introduction_single.tex

\section{Introduction} 

In this paper, we focus on the numerical solution of optimization
problems that consist in
minimizing the sum of an $\ell_2$-fitting term and a sparsity-inducing regularization term.
Such problems are of the form:
\begin{align}
\label{eqn:sp_lsproblem}
\mathcal{P}(\lambda, \spreg, \D ,\obs):  \pvaropt \triangleq  \argmin{\pvar} \frac{1}{2}\| \D \pvar -\obs \| _2^2 + \lambda \spreg (\pvar),
\end{align}
where $\obs \in \R^{\dimS}$ is an observation; $\D \in 
\R^{N\times K}$ with $N\leq K$ is a matrix called dictionary; $\spreg : \R^{\nAt} \rightarrow  
\R_+$ is a convex sparsity-inducing regularization function; and $\lambda>0$ is a parameter that 
 governs the tradeoff between data fidelity and regularization. 
Various convex and non-smooth functions $\spreg$ may induce the sparsity of the solution $\pvaropt$.
Here, we consider two instances of problem~\eqref{eqn:sp_lsproblem}, 
the \lasso~\cite{Tibshirani94regressionshrinkage} and the \glasso~\cite{Yuan2006}, which 
differ from each other in the choice of the regularization $\spreg$. 
A key challenge is to handle~\eqref{eqn:sp_lsproblem} when both $N$ and $K$
may be large, which occurs in many real-world applications including 
denoising~\cite{Chen1998a}, inpainting~\cite{Elad2005} or classification~\cite{Wright2009a}.
Algorithms relying on first-order information only, \ie gradient-based procedures~\cite{Beck2009,ComWaj05,Daubechies2004,Wright2009}, are
particularly suited to solve these problems, as second-order based methods
---\eg using the Hessian--- imply too computationally demanding iterations.
In the $\ell_2$ data-fitting case the gradient relies on the application of the operator $\D$ and its transpose $\D^T$.
We use this feature to define \emph{first order} algorithms as those based on
the application of $\D$ and $\D^T$. The definition extends to primal-dual 
algorithm~\cite{Arrow64,Chamb2011}.

Accelerating these algorithms remains challenging:
even though they provably have fast convergence~\cite{Nesterov1983,Beck2009}, the multiplications by $\D$ and
$\D^T$ in the optimization process are a bottleneck in their computational efficiency, which is thus governed by the
dictionary size. We are interested in the general case where no fast transform is associated with $\D$. This occurs
for instance with exemplar-based or learned dictionaries. Such accelerations are even more needed during the 
process of learning high-dimensional dictionaries, which requires solving many problems of the form~\eqref{eqn:sp_lsproblem}.
\subsubsection*{ Screening Tests} The convexity of the objective function suggests to use the theory of convex duality~\cite{Boyd94convex} 
to understand and solve such problems.
In this context, strategies based on {\em screening
tests}~\cite{Pelckmans2012,ElGhaoui2010,Tibshirani2012,Wang2012,Xiang2011,Xiang2012}
have recently been proposed to reduce the computational cost by considering properties of the dual optimum. 
Given that the sparsity-inducing regularization $\spreg$ entails an optimum $\pvaropt$ that may contain 
many zeros, a screening test is a method aimed at locating a subset of such zeros. It is formally defined as 
follows:

\begin{defn}[Screening test]\label{def:scrTest}
Given problem $\mathcal{P}(\pen, \spreg, \D, \obs)$ with solution $\pvaropt$, a boolean-valued function
${\ST :  [1 \ldots \nAt] \rightarrow \{0,1\} }$ is a screening
test if and only if:
\begin{align}
\forall i \in  [1 \ldots \nAt], \ST(i)=1 \Rightarrow \coef{\pvaropt}{i} = 0 \label{prop:screenTest}.
\end{align}
\end{defn}
We assume that solution $\pvaropt$ is unique ---\eg, it is true with probability one 
for the \lasso if $\D$ is drawn from a continuous distribution~\cite{Tibshirani2013}.
In general, a screening test cannot detect all zeros in $\pvaropt$, that is why relation~\eqref{prop:screenTest}
is not an equivalence. An efficient screening test locates many zeros among those of $\pvaropt$.
From a screening test $\ST$ a {screened dictionary} $\D_0 = \D \recMat$ is defined,
where the matrix $\recMat$ removes from $\D$ the \emph{inactive} atoms corresponding 
to the zeros located by the screening test $\ST$. 
$\recMat$ is built by removing, from the $K \times K$ identity matrix, the columns corresponding to screened atoms 
\ie columns indexed by $i \in [1 \ldots \nAt]$ whenever verifies ${\ST(i) =1}$.
Property~\eqref{prop:screenTest} implies that $\pvaropt_0$, the solution of problem $\mathcal{P}(\lambda,\spreg,\D_0 ,\obs)$,
is exactly the same as $\pvaropt$ the solution of $\mathcal{P}(\lambda,\spreg,\D ,\obs)$ up to 
inserting zeros at the locations of the removed atoms: $\pvaropt = \recMat \pvaropt_0$.
Any optimization procedure solving problem $\mathcal{P}(\pen, \spreg,\D_0,\obs)$ with the screened dictionary $\D_0$
therefore computes the solution of $\mathcal{P}(\lambda, \spreg,\D ,
\obs)$ at a lower computational cost.
Algorithm~\ref{alg:staticscreen} depicts the commonly used strategy~\cite{ElGhaoui2010,Xiang2011} to obtain an
 algorithmic acceleration using a screening test; {it rests upon two steps:
i) locate some zeros of $\pvaropt$ thanks to a {screening test}
and construct the {\em screened} dictionary $\D_0$ and ii) solve 
 $\mathcal{P}(\lambda, \spreg, \D_0,\obs)$ using the smaller dictionary $\D_0$.


\subsubsection*{Dynamic Screening}

We propose a new screening principle called \emph{dynamic screening} in order to 
reduce even more the computational cost of first-order algorithms.
We take the aforementioned concept of {screening test} one step further,
and improve existing
screening tests by \emph{embedding} them in the iterations of first-order algorithms. 
We take advantage of the computations made during the optimization procedure to perform a 
new {screening test} at each iteration 
 with a \emph{negligible} computational overhead, and we consequently \emph{dynamically} reduce the size of $\D$.
For a schematic comparison, the existing \emph{static} screening and the proposed \emph{dynamic} screening are
sketched in Algorithms~\ref{alg:staticscreen} and \ref{alg:dynscreen}, respectively. One may observe that with the
dynamic screening strategy, the dictionary $\D_{\it}$ used at each iteration $\it$ is possibly smaller and smaller 
thanks to successive screenings.



\begin{figure}[t] 
\center
\begin{minipage}[t]{.4\textwidth}

\vspace{-0.4cm}
\begin{algorithm}[H]

\algsetblockx[chose]{noindentloop}{noindentendloop}{9}{+0.5em}
		[1][]{ \raggedright \hspace{-1.1em}\textbf{loop #1}}
		[1][]{ \raggedright \hspace{-1.1em}\textbf{end loop}}

\caption{\\Static screening strategy }
\label{alg:staticscreen} 
\small
\begin{algorithmic} 
\State \hspace{-1.1em} $\D_0 \gets$ Screen $\D$
\noindentloop{t}
\State $\pvar_{\it+1}\leftarrow$Update $\pvar_{\it}$ using $\D_0$
\noindentendloop 
\end{algorithmic}
\end{algorithm} 
\end{minipage}
\begin{minipage}[t]{.4\textwidth}
\vspace{-0.4cm}
\begin{algorithm}[H]
\caption{\\Dynamic screening strategy}
\algsetblockx[chose]{noindentloop}{noindentendloop}{9}{+0.5em}
		[1][]{ \raggedright \hspace{-1.1em}\textbf{loop #1}}
		[1][]{ \raggedright \hspace{-1.1em}\textbf{end loop}}
\label{alg:dynscreen}
\small
\begin{algorithmic}
\State \hspace{-.8em}$\D_{0}\gets\D$
\noindentloop{t}
\State  $\pvar_{\it+1}\gets$ Update $\pvar_{\it}$ using $\D_{\it}$
\State  $\D_{\it+1}\gets$ Screen $\D_{\it}$ using $\pvar_{\it+1}$
\noindentendloop
\end{algorithmic}
\end{algorithm}
\end{minipage}
\end{figure}

\subsubsection*{Illustration}
We now present a brief illustration of the dynamic screening principle in action, dedicated to the impatient
 reader who wishes to get a good grasp of the approach without having to enter the mathematics in too much detail.
The dynamic screening principle is illustrated in Figure~\ref{fig:screenprogress} 
in the particular case of the combined use of ISTA~\cite{Daubechies2004} and of a new, dynamic version
 of the SAFE screening test~\cite{ElGhaoui2010} to solve the \lasso problem, 
which is~\eqref{eqn:sp_lsproblem} with $\spreg(\pvar) = \normI{\pvar}$.
The screening effect is illustrated through the evolution of the size of the dictionary,
\ie, the number of atoms remaining in the screened dictionary $\D_{\it}$.
In this example the observation $\obs$ and all the $\nAt=50000$ atoms are vectors drawn
uniformly and independently on the unit sphere in dimension $\dimS=5000$ of $\D$,
and we set $\lambda = 0.75\times\normInf{\D^T\obs}$.
Here, the actual screening begins to have an effect at iteration $t=20$. In about ten iterations, the dictionary is 
dynamically screened down to $5 \%$ of its initial size and the computational cost of each iteration gets 
lower and lower. Then, the last iterations are performed with a reduced computational cost. 
Consequently, the total running time equals 4.6 seconds while it is 11.8 seconds if no screening is 
used. One may also observe that the screening test is inefficient in the first $20$ iterations. 
In particular, the dynamic screening at the very first iteration is 
strictly equivalent to the state-of-the-art (static) screening. Static screening would have been of no use in this 
case.


\begin{figure}[t]
\centering
\includegraphics[width=0.55 \textwidth]{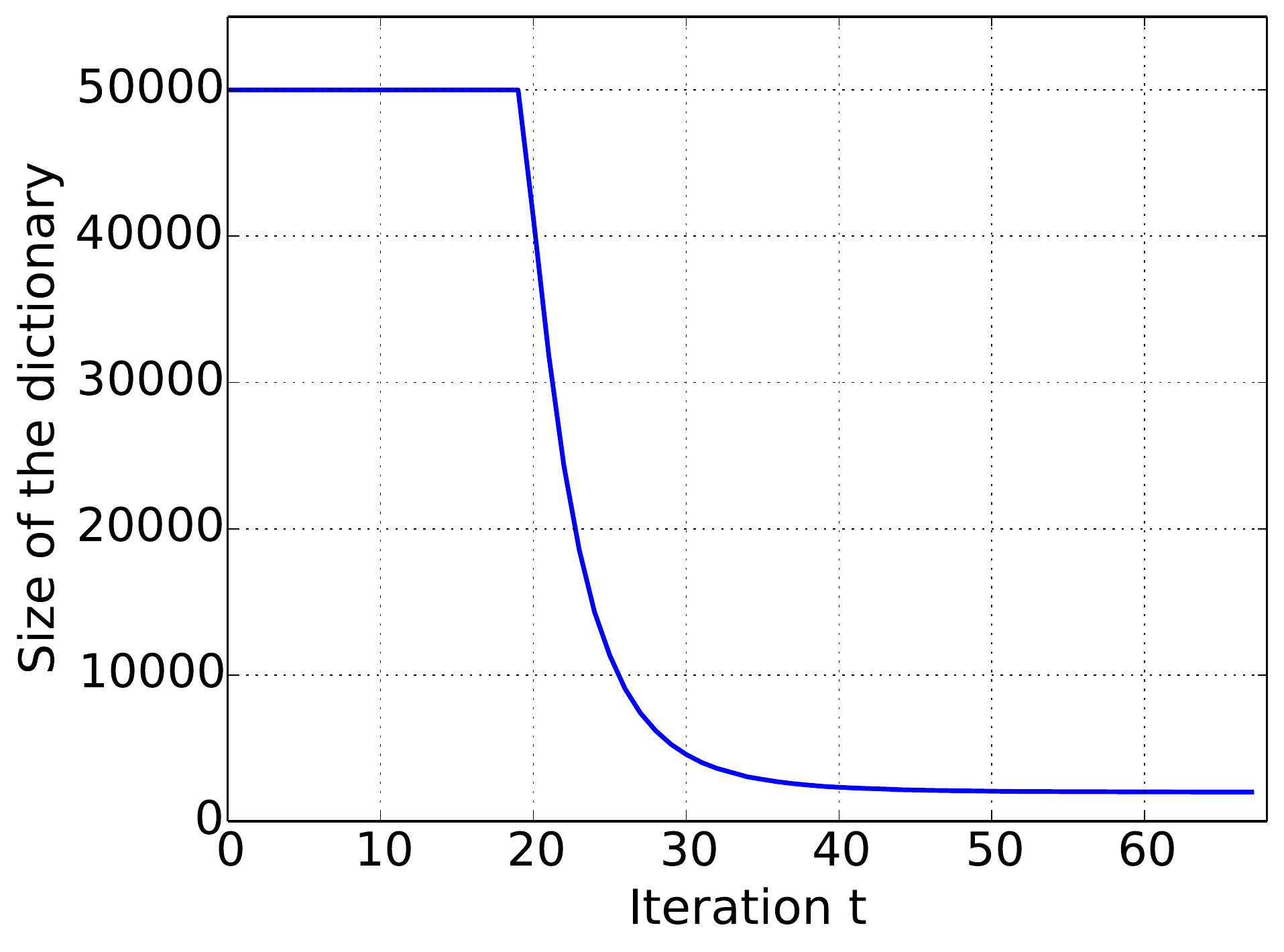} 
\caption{Size of the dictionary $\D_{\it}$ (number of atoms) as a function of the iteration $\it$ in a basic
 dynamic-screening setting, starting with a dictionary with $\nAt=5000$ atoms.}\label{fig:screenprogress}
\end{figure}
\subsubsection*{Contributions}
This paper is an extended version of~\cite{Bonnefoy2014}. Here we propose new screening tests for the \glasso 
and give an unified formulation of the dynamic screening principle for both \lasso and \glasso, which improves screening tests for both problems.
The algorithmic contributions and 
related theoretical results are introduced in Section~\ref{sec:gal_dynscreen}. First, the dynamic screening 
principle is formalized in a general algorithmic scheme (Algorithm~\ref{alg:GalDynScreenDetail}) and its 
convergence is established (Theorem~\ref{thm:DynScrCvg}). Then, we show how to instantiate this scheme 
for several first-order algorithms (Section~\ref{sec:first-order}), as well as for the two considered problems
the \lasso and \glasso (Sections~\ref{sec:dynscr_lasso} and~\ref{sec:dynscr_glasso}). We adapt existing tests to make them 
dynamic for the \lasso and propose new screening tests for the \glasso. A turnkey instance of the 
proposed approach is given in Section~\ref{sec:all_in_one}. Finally, the computational complexity of the 
dynamic screening scheme is detailed and discussed in Section~\ref{sec:optimandscreen_comput}.
In Section~\ref{sec:simulations}, 
experiments show how the dynamic screening \emph{principle} 
significantly reduces the computational cost of first-order optimization algorithms for a large range of 
problem settings and algorithms.
We conclude this paper by a discussion in Section~\ref{sec:discussion}. 
All proofs are given in Appendix.

%% file: General_dynScreen.tex
\section{Generalized Dynamic Screening}\label{sec:gal_dynscreen}

\paragraph*{Notation and definitions}

$\dict\triangleq\left[\atom_1,\ldots,\atom_\nAtoms\right]\in\R^{\dimSig\times \nAtoms}$ denotes a {\em dictionary} and
${\fullset \triangleq \left\lbrace 1,\ldots,K\right\rbrace}$ denotes the set of integers indexing the columns, or \emph{atoms}, of $\dict$. The $i$-th component of $\pvar$ is denoted by $\coef{\pvar}{i}$.
For a given set  $\mathcal{I} \subset \fullset$, $\card{\mathcal{I}}$ is the cardinal of $\mathcal{I}$, 
$\comp{\mathcal{I}}\triangleq\fullset\backslash\mathcal{I}$
is the complement of $\mathcal{I}$ in $\fullset$ and 
$\D_{[\mathcal{I}]} \eqdef \left[ \atom_i\right]_{i\in\mathcal{I}}$
denotes the sub-dictionary composed of the atoms indexed by elements of $\mathcal{I}$.
The notation extends to vectors: $\grindex{\pvar}{\mathcal{I}} \eqdef [\coef{\pvar}{i}]_{i \in \mathcal{I}}$.
Given two index sets $\mathcal{I},\mathcal{J}\subset \fullset$ and a matrix $\mymat{M}$, 
 $\grindex{\mymat{M}}{\mathcal{I},\mathcal{J}} \eqdef 
[\mymat{M}(i, j)]_{(i,j) \in\mathcal{I} \times \mathcal{J}}$ denotes the sub-matrix of $\mymat{M} $ obtained by selecting the rows and columns indexed by elements in $\mathcal{I}$ and $\mathcal{J}$, respectively.
We denote primal variables vectors by $\pvar \in \R^\nAt$ and dual variables vectors by $\dvar \in \R^\dimS$. 
We denote by $[r]_{a}^b \eqdef  \max(\min(r,b),a)$ the projection of $r$ onto the segment $[a,b]$.
Without loss of generality, the observation $\obs$ and the atoms $\atom_i$ are assumed to have unit $\ell_2$ norm. 
For any matrix $\mymat{M}$, $\normM{\mymat{M}}$ denotes its spectral norm, \ie, its largest singular value.


\subsection{Proposed general algorithm with dynamic screening}
Dynamic screening is dedicated to accelerate the computation of the solution of problem~\eqref{eqn:sp_lsproblem}. It is presented here in a general way.

Let us consider a problem $\mathcal{P}(\lambda, \spreg, \D ,\obs)$ as defined by eq.~\eqref{eqn:sp_lsproblem}.
First-order algorithms are iterative optimization procedures that may be resorted to solving
$\mathcal{P}(\pen, \spreg, \D, \obs)$.
Based only on applications of $\D$ and $\D^T$,
they build a sequence of iterates $\pvar_{\it}$ that
converges, either in terms of objective values or the iterate itself, to the solution of the problem.
In the following, we use the update step function $\step( \cdot)$ as a generic notation to refer to any first-order algorithm.
The optimization procedure might be formalized as the update 
$$ ( \multipvar_{\it}, \dvar_{\it} , \auxvar_{\it} ) \gets \step (\multipvar_{\it-1}, \dvar_{\it-1}, 
\auxvar_{\it-1}, \D)$$
of several variables. Matrix $\multipvar_{\it}$ is composed of one or several columns 
that are primal variables and from which one can extract $\pvar_{\it}$;
vector $\dvar_{\it}$ is an updated variable of the dual space $\R^\dimS$ ---in the sense 
of convex duality (see~\ref{sec:dynscr_lasso} equation \eqref{eqn:dual} for details); and $\auxvar_{\it}$ is a list of updated auxiliary scalars in $\R$. 

Algorithm~\ref{alg:GalDynScreenDetail} makes explicit the use of the introduced notation $\step(\cdot)$ in the
general scheme of the dynamic screening principle. 
The inputs are: the data that characterize problem $\mathcal{P}(\lambda, \spreg, \D ,\obs)$; the update 
function $\step( \cdot)$ related to the first-order algorithm to be accelerated; an initialization $\multipvar_0$; 
and a family of screening tests $\left\lbrace\ST_{\dvar}\right\rbrace_{\dvar \in \R^{\dimS}}$ in the 
sense\footnote{Algorithm~\ref{alg:GalDynScreenDetail} uses notation with screening tests indexed by a dual point 
$\dvar$ however the proposed Algorithm~\ref{alg:GalDynScreenDetail}  is valid in a more general 
case for any screening test $\ST$ that may be designed. 
We use the notation $\ST_{\dvar}$ since in this paper ---and in the literature--- screening tests are based on a dual point.}
 of Definition~\ref{def:scrTest}. Iteration $\it$ begins with an update step at line~\ref{inAlg:optim}. It is followed by a screening 
stage in which a screening test $\ST_{\dvar_{\it}}$ is computed using the dual point $\dvar_{\it}$ obtained during the update. 
As shown at line~\ref{inAlg:ScrTest}, it enables to detect new inactive atoms and to update index set $\scr_{\it}$ that gathers 
all atoms identified as inactive so far. This set is then used to screen the dictionary and the 
primal variables at lines~\ref{inAlg:screenDict} and~\ref{inAlg:screenPrimal} using the screening matrix 
$\grindex{\Id}{\comp{\scr}_{\it},\comp{\scr}_{\it-1}}$ ---obtained by removing columns $\scr_{\it-1}$ and rows 
${\scr_{\it}}$ from the $\nAt\times\nAt$-identity matrix--- and its transpose. Thanks to successive screenings, the dimension shared by the
primal variables and the dictionary is decreasing and the optimization update can be computed in the reduced dimension 
$\card{\comp{\scr}_{\it}}$, at a lower computational cost. The acceleration is efficient 
because lines~\ref{inAlg:ScrTest} to~\ref{inAlg:screenPrimal} have negligible computation impact, as 
shown in Section~\ref{sec:optimandscreen_comput} and assessed experimentally in 
Section~\ref{sec:simulations}.

\begin{algorithm}[tbh]
\begin{algorithmic}[1]
\Require $\D, \obs, \lambda, \spreg, \multipvar_0, \text{ screening test } \ST_{\dvar}$
	 for any $\dvar \in \R^{\dimS}$ and first-order update $\step(\cdot) $.
\State $\D_0 \gets \D$, $\scr_0 \gets \emptyset, \it \gets 1, \bar \multipvar_{0} \gets \multipvar_{0} $ 
\While{stopping criteria on $\multipvar_\it$}
 	\State ........................  Optimization Update  ........................
	\State $ ( \multipvar_{\it}, \dvar_{\it} , \auxvar_{\it} ) \gets 
		\step (\bar \multipvar_{\it-1}, \dvar_{\it-1}, \auxvar_{\it-1}, \D_{\it-1})$ 
		\label{inAlg:optim}
	\State ................................  Screening  ................................. 
	\State $\scr_{\it} \gets \{ i \in \fullset, \ST_{\dvar_{\it}}(i)\} \cup \scr_{\it-1}$ 
	\label{inAlg:ScrTest}
	\State $\D_{\it} \gets  \D_{\it-1} \grindex{\Id}{\comp{\scr}_{\it-1},\comp{\scr}_{\it} }$
	 \label{inAlg:screenDict}
	\State $\bar \multipvar_{\it} \gets 
		\grindex{\Id}{\comp{\scr}_{\it},\comp{\scr}_{\it-1} } \multipvar_{\it}$ 
	\label{inAlg:screenPrimal}
	\State $\it \gets \it+1$ 
\EndWhile\\
\Return $\multipvar_{\it}$
\end{algorithmic}
\caption{General algorithm with dynamic screening}
\label{alg:GalDynScreenDetail}
\end{algorithm}

Since the optimization update at line~\ref{inAlg:optim} is performed in a reduced dimension, 
the iterates generated by the algorithm with dynamic screening may differ from those of the 
base first-order algorithm. The following results states that the proposed general algorithm
 with dynamic screening preserves the convergence to the global minimum of problem~\eqref{eqn:sp_lsproblem}. 

\begin{thm}\label{thm:DynScrCvg}
Let $\spreg$ be a convex and sparsifying regularization function and $\step(\cdot)$ the update function of an iterative algorithm.
If $\step(\cdot)$ is such that for any $\D, \obs, \pen$, the sequence of iterates given by $\step(\cdot)$ converges
to the solution of $\mathcal{P}(\pen,\spreg,\D,\obs)$,
then, for any family of screening tests
$\{\ST_{\dvar}\}_{\dvar \in \R^\dimS }$, the general algorithm with dynamic screening (Algorithm~\ref{alg:GalDynScreenDetail})
converges to the same global optimum of problem $\mathcal{P}(\pen,\spreg,\D,\obs)$.
\end{thm} 
\begin{proof}
Since $ \forall \dvar \in \R^{ \dimS} , \ST_{\dvar}$ is a screening test, problems $\mathcal{P}( 
\lambda,\spreg,\D ,\obs)$ and $\mathcal{P}(\lambda,\spreg,\D_\it ,\obs)$ for all $\it\geq 0$ have 
the same solution.
The sequence $\{\scr_\it\}_{\it \geq 0}$ of located zeros at time $\it$ is inclusion-wise non-decreasing and 
upper bounded by the {set} of zeros in {$\pvaropt$} the solution of 
$\mathcal{P}(\lambda, \spreg,\D ,\obs)$, {so} the sequence converges in a finite number of 
iterations $\it_0$.
Then ${\forall \it\geq \it_0 , \D_{\it_0}=\D_\it}$ and the existing convergence proofs of the first-order algorithm with update $\step(\cdot)$ apply.
\end{proof}

\subsection{Instances of algorithm with dynamic screening}
\label{sec:instances}

The general form of Algorithm~\ref{alg:GalDynScreenDetail} may be instantiated for various algorithms, 
problems and screening tests. The general form with respect to first-order algorithms is instantiated 
for a number of possible algorithms in Section~\ref{sec:first-order}. 
The algorithm may be applied to solve the \lasso and \glasso problems through the choice of the 
regularization 
$\spreg$. Instances of screening tests $\ST_{\dvar}$ associated to each problem are described in 
Sections~\ref{sec:dynscr_lasso} and~\ref{sec:dynscr_glasso} respectively. 
Proofs are postponed to the appendices for readability purposes.

\subsubsection{First-order algorithm updates}
\label{sec:first-order}

The dynamic screening principle may accelerate many first-order algorithms.
Table~\ref{tab:optimstep} specifies how to use Algorithm~\ref{alg:GalDynScreenDetail} for several first-order algorithms, namely ISTA~\cite{Daubechies2004}, TwIST~\cite{Bioucas2007}, SpaRSA~\cite{Wright2009}, FISTA\cite{Beck2009} and Chambolle-Pock~\cite{Chamb2011}. 
This specification consists in defining the primal variables $\multipvar_{\it}$, the dual variable $\dvar_{\it}$, 
the possible additional variables $\auxvar_{\it}$ and the update $\step(\cdot)$ used at line~\ref{inAlg:optim}.

Table~\ref{tab:optimstep} shows two important aspects of first-order algorithms: first, the notation is 
general and many algorithms may be formulated in this way; second, every
$\step(\cdot)$ has a computational complexity in $\mathcal{O}( \nAt \dimS )$ per iteration.

\begin{table*}[tbh]
\begin{center}
\small
\begin{tabular}{lll}
\toprule
\multicolumn{1}{c}{\multirow{2}{*}{Algorithm}} 
& \multicolumn{1}{c}{\multirow{2}{*}{Nature of $\{\multipvar_{\it}, \auxvar_{\it}\}$}} 
& \multicolumn{1}{c}{ Optimization \emph{update} }\\
& & \multicolumn{1}{c}{$ \{ \multipvar_{\it}, \dvar_{\it} \} \gets \step (\multipvar_{\it-1}, \dvar_{\it-1}, \auxvar_{\it-1}, \D)$}\\\midrule

\multirow{3}{*}{ISTA~\cite{Daubechies2004}}   & \multirow{3}{*}{$\multipvar_{\it}  \eqdef \pvar_\it ,\auxvar_{\it} \eqdef L_\it$} & $\dvar_{\it} \gets \D\pvar_{\it-1} - \obs$ \\
& & $\pvar_{\it} \gets  \prox{\spreg}{\frac{\lambda}{L_\it}}{\pvar_{\it-1} - \frac{1}{L_\it} \D^T \dvar_{\it}}$\\ 
& &$L_\it$ is set with the backtracking rule see~\cite{Beck2009}\\ \midrule
\multirow{2}{*}{TwIST~\cite{Bioucas2007}}  & \multirow{2}{*}{$\multipvar_{\it}  \eqdef [\pvar_\it, \pvar_{\it-1}],\auxvar_{\it} \eqdef \emptyset$} & $\dvar_{\it} \gets \D\pvar_{\it-1} - \obs$\\
&& $\pvar_{\it} \gets  (1-\alpha)\pvar_{\it-2} + (\alpha-\beta)\pvar_{\it-1}+ \beta \prox{\spreg}{\lambda}{\pvar_{\it-1} - \D^T \dvar_{\it}}$ \\
\midrule
SpaRSA~\cite{Wright2009} & \multirow{1}{*}{$\multipvar_{\it}  \eqdef [\pvar_\it],\auxvar_{\it} \eqdef L_{\it}$} &  Same as ISTA except that $L_\it$ is set with the Brazilai-Borwein rule~\cite{Wright2009} \\ \midrule
\multirow{5}{*}{FISTA~\cite{Beck2009}} & \multirow{5}{*}{$\multipvar_{\it}  \eqdef [\pvar_\it, \pvaru_{\it}], \auxvar_{\it} \eqdef \left(l_{\it},L_\it\right)$} & $\dvar_{\it} \gets \D \pvaru_{\it -1 } -\obs$\\
&&$\pvar_{\it} \gets \prox{\spreg}{\lambda/L_\it}{\pvaru_{\it-1} - \frac{1}{L_\it}\D^T \dvar_{\it}} $\\
&&$l_{\it} \gets \frac{1+\sqrt{1+4l_{\it-1}^2}}{2}$\\
&&$ \pvaru_{\it} \gets \pvar_{\it} + (\frac{l_{\it-1}-1}{l_{\it}})(\pvar_{\it}-\pvar_{\it-1}) $\\ 
&&$L_\it$ is set with the backtracking rule see~\cite{Beck2009}\\ \midrule 
\multirow{5}{*}{Chambolle-Pock~\cite{Chamb2011}} & \multirow{5}{*}{$\multipvar_{\it} \eqdef [\pvar_{\it},{\pvaru}_{\it}],\auxvar_{\it} \eqdef \left(\tau_\it,\sigma_{\it} \right) $} 
& $\dvar_{\it} \gets \frac{1}{1+\sigma_{\it-1}}(\dvar_{\it-1} + \sigma_{\it-1}(\D {\pvaru}_{\it-1} - \obs))$\\
&&$ \pvar_{\it} \gets \prox{\spreg}{\lambda\tau_{\it-1}}{\pvar_{\it-1}- \tau_{\it-1} \D^T \dvar_{\it}} $\\
&&$ \varphi_\it \gets \frac{1}{\sqrt{1+2\gamma \tau_{\it-1}}}$\\ 
&& $ \tau_{\it} \gets \varphi_\it \tau_{\it-1};\sigma_{\it} \gets \frac{\sigma_{\it-1}}{\varphi_\it} $\\
&&$ {\pvaru}_{\it} \gets \pvar_{\it} + \varphi_{\it}(\pvar_{\it}-\pvar_{\it-1}) $ \\ 
\bottomrule
\end{tabular}
\caption{\emph{Updates} for first-order algorithms.}
\label{tab:optimstep}
\end{center} 
\vspace{-.4cm}
\end{table*}

Table~\ref{tab:optimstep} makes use of proximal operators 
${ \prox{\spreg}{\pen}{\pvar} \triangleq \min_{\mathbf{z}}\spreg\left(\mathbf{z}\right)
 +\frac{1}{2\pen}\normII{\pvar-\mathbf{z}}^2}$ 
that handle the non-smoothness of the objective function introduced by the regularization $\spreg$.
Beyond the subsequent definitions of the proximal operators for the \lasso (see eq.~\eqref{eqn:softTh}) and the \glasso (see eq.~\eqref{eqn:gr_softTh}), we refer the interested reader to~\cite{ComWaj05}
 for a full description of proximal methods.


\subsubsection{Dynamic screening for the \lasso}\label{sec:dynscr_lasso}

Let us first recall the \lasso problem before giving the screening tests $\ST_{\dvar}$ that may
be embedded in the corresponding optimization procedure.

\paragraph{The \lasso~\cite{Tibshirani94regressionshrinkage}}
The \lasso problem uses the $\ell_1$-norm penalization to enforce a sparse solution. 
The \lasso is exactly~\eqref{eqn:sp_lsproblem} using ${\spreg(\pvar) = \normI{ \pvar }}$: 
\begin{align}\label{eqn:lasso}
\mathcal{P}^\text{\lasso}(\lambda, \D ,\obs) : \argmin{\pvar} \dfrac{1}{2}\normII{\D \pvar - \obs}^2 + \pen \normI{ \pvar }.
\end{align}
The proximal operator of the $\ell_1$-norm is the so-called soft-thresholding operator: 
\begin{align}
\prox{\lasso}{t}{\pvar} \eqdef \sign( \pvar ) \max(\abs{\pvar}-t,0) \label{eqn:softTh}.
\end{align}
Screening tests~\cite{ElGhaoui2010, Xiang2012, Xiang2011} rely on the dual formulation of the \lasso problem:
\begin{subequations}\label{eqn:dual}
\begin{align} 
          \dvaropt \triangleq & \argmax{\dvar}  \frac{1}{2}\normII{\obs}^2-\frac{\lambda^2}{2}\normII{\dvar-\frac{\obs}{\lambda}}^2     \label{eqn:dual_unconstr}
         \\ & \text{ s.t. }  \forall i \in \fullset,  \abs{\dvar^T\atom_i}\leq 1 .\label{eqn:dual_constr}
\end{align}
\end{subequations}
A dual point $\dvar\in \R^{\dimS}$ is said \textit{feasible} for the \lasso if it complies with constraints~\eqref{eqn:dual_constr}.

From the convex \emph{optimality conditions}, solutions of the \lasso problem~\eqref{eqn:lasso} and its dual~\eqref{eqn:dual}, 
$\pvaropt$ and $\dvaropt$ respectively, are necessarily linked by:
      \begin{align}
          \obs & = \D\pvaropt +\lambda\dvaropt \text{ and } \forall i \in \fullset,  
	 \begin{cases}          
         | \dvaropt^T\atom_i |\leq 1&  \text{ if }  \coef{\pvaropt}{i}= 0,\\
          | \dvaropt^T\atom_i |=1&  \text{ if }  \coef{\pvaropt}{i}\neq 0.
          \end{cases}
           \label{eqn:primaldual}
      \end{align}

We define $\lambda_* \eqdef  \normInf{\D^T\obs}$. 
If $\lambda >\lambda_*$, the solution is trivial and is derived from the most simple screening test that may be designed, which screens out \emph{all} atoms.
Indeed, starting from the fact that $\dvaropt=\obs/\lambda$ is the solution of the dual problem ---it is feasible and 
maximizes~\eqref{eqn:dual_unconstr}---, we have $\forall i, | \atom_i^T \dvaropt |=| \obs^T\atom_i |/\lambda 
\leq \lambda_*/\lambda < 1$ so that the optimality conditions impose that $\coef{\pvaropt}{i}=0$. In other words, 
we can screen the whole dictionary before entering the optimization procedure. In the following, we will focus on the
 non-trivial case $\lambda \in \left]0,\lambda_*\right] $.

\paragraph{Screening tests for the \lasso}

The screening tests presented here have been proposed initially in the static perspective 
in~\cite{ElGhaoui2010, Xiang2012, Xiang2011}.
They use relation~\eqref{eqn:primaldual} to locate some inactive atoms $\atom_{i}$ for which $\abs{\atom_i^T \dvaropt} < 1$. 
The quantity is not directly accessible as the optimal is not known, thus the base concept of screening test is to geometrically construct a region $\mathcal{R}$ 
that is known to contain the optimal $\dvaropt$ so that the upper-bound $\max_{\dvar \in \mathcal{R}} 
\abs{\atom_{i}^T \dvar} \geq \abs{\atom_{i}^T \dvaropt}$ gives a sufficient condition for atom $\atom_i$ to 
be inactive: $\max_{\dvar \in \mathcal{R}}\abs{\atom_{i}^T \dvar}<1 \Rightarrow \coef{\pvaropt}{i} =0 $.
In particular the previous maximization problem admits a closed form solution when $\mathcal{R}$ is a sphere or a dome.

The SAFE test proposed by L. El Ghaoui et al. in~\cite{ElGhaoui2010} 
is derived by constructing a sphere from any dual point $\dvar$. 
Xiang et al. in~\cite{ Xiang2012, Xiang2011} improved it when the particular dual point ${\obs}$ is used.
We propose here a homogenized formulation relying on any dual point $\dvar$ for each of the three screening tests,
generalizing~\cite{Xiang2012, Xiang2011} to any dual point, thereby fitting them for use in a dynamic setting.
We present these screening tests through the following Lemmata, in order of increasing description complexity and screening efficiency.
For more details on the construction of regions $\mathcal{R}$ and the solution of the maximization problem
please see the references or proofs in Appendix.

\begin{lem}[The $\safe{}$ screening test~\cite{ElGhaoui2010}]\label{thm:SAFE}
For any ${ \dvar \in \R^{\dimS} }$, the following function $\ST^{\safe{}}_{\dvar}$ is a screening test for $\mathcal{P}^\text{\lasso}(\lambda, \D ,\obs)$: 
\begin{align}
 \ST^{\safe{}}_{\dvar} : \fullset &\rightarrow \{0,1\}   \nonumber \\
 i & \mapsto \left \llbracket (1-\abs{\atom_{i}^T\sphcent} )> \sphradf{\dvar} \right \rrbracket \label{eqn:safe}
\end{align}
where $\sphcent \eqdef \frac{\obs}{ \lambda }$, $\sphradf{\dvar} \eqdef \normII{\frac{\obs}{\pen} - \mu \dvar}$
 and $\mu \eqdef \left[\frac{\dvar^T\obs}{\lambda \|\dvar \|^2_2}\right]^{\|\D^T \dvar \|_{\infty}^{-1}}_{-\|\D^T \dvar \|_{\infty}^{-1}}$.
\end{lem}

The notation $\llbracket \mathrm{P} \rrbracket$ in~\eqref{eqn:safe} means that we take the boolean value of the proposition $\mathrm{P}$
as the output of the screening test.
We also recall that ${ [r]_{a}^b \eqdef  \max(\min(r,b),a) }$ denotes the projection of $r$ onto the segment
$[a,b]$. Lemma~\ref{thm:SAFE} is exactly El Ghaoui's SAFE test.

The screening test ST3~\cite{Xiang2011} is a much more efficient test than SAFE, especially when $\lambda_*$
is high. We extend it in the following Lemma so that it can be used for dynamic screening.

\begin{lem}[The Dynamic ST3: $\hsst{}$]\label{thm:HSST}
For any ${\dvar \in \R^{\dimS} }$, the following function $\ST^{\hsst{}}_{\dvar}$ is a screening test for $\mathcal{P}^\text{\lasso}(\lambda, \D ,\obs)$: 
\begin{align}
 \ST^{\hsst{}}_{\dvar} : \fullset &\rightarrow \{0,1\}   \nonumber \\
 i & \mapsto \left \llbracket ( 1-\abs{\atom_{i}^T\sphcent} ) > \sphradf{\dvar} \right  \rrbracket
 \label{eqn:DST3}
\end{align}
where $\sphcent \eqdef \frac{\obs}{ \lambda }- \left (\frac{\lambda_*}{\lambda}-1\right ) \atom_*$, 
$\sphradf{\dvar} \eqdef \sqrt{ \normII{ \mu_{} \dvar -\frac{\obs}{ \lambda } }^2- 
\left (\frac{\lambda_*}{\lambda}-1\right )^2}$,
${\mu \eqdef \left[\frac{\dvar^T\obs}{\lambda \|\dvar \|^2_2}\right]^{\|\D^T \dvar \|_{\infty}^{-1}}_{-\|\D^T \dvar \|_{\infty}^{-1}}}$ 
\\
and $\atom_* \eqdef \argmax{\atom\in \{ \pm \atom_i \}_{i=1}^K}  \atom^T\obs$.
\end{lem}

When applied with $\dvar = {\obs}$ this screening test is exactly the ST3~\cite{Xiang2011}. 
Further improvements have been proposed in the Dome test~\cite{Xiang2012} for which we also
 propose an extended version appropriate for dynamic screening.

\begin{lem}[The Dynamic Dome Test: $\hdst{}$]\label{thm:HDST}
For any ${ \dvar \in \R^{\dimS} }$, the following function $\ST^{\hdst{}}_{\dvar}$ is a screening test for $\mathcal{P}^\text{\lasso}(\lambda, \D ,\obs)$: 
\begin{align}
 \ST^{\hdst{}}_{\dvar} : \fullset &\rightarrow \{0,1\}   \nonumber \\
 i & \mapsto  \left \llbracket Q^l_{\dvar}l(\atom_*^T\atom_i) < \pvar^T \atom_i < Q^u_{\dvar}(\atom_*^T\atom_i) \right \rrbracket
\end{align}
where 
\begin{align}
Q^l_{\dvar}(t)& \eqdef
\begin{cases}
(\lambda_* - \pen)t - \pen +\pen \sphradf{\dvar} \sqrt{1-t^2}, & \text{if } t \leq \lambda_*\\
-(\pen - 1 + \pen/ \lambda_*), & \text{if } t > \lambda_*
\end{cases} \label{eqn:dome1}\\
Q^u_{\dvar}(t)&  \eqdef
\begin{cases}
(\pen - 1 + \pen/ \lambda_*), & \text{if } t < -\lambda_* \\
(\lambda_* - \pen)t + \pen -\pen \sphradf{\dvar} \sqrt{1-t^2}, & \text{if } t \geq -\lambda_*
\end{cases}\label{eqn:dome2}
\end{align}
 $\sphradf{\dvar} \eqdef \sqrt{ \normII{ \mu_{} \dvar -\frac{\obs}{ \lambda } }^2
 - \left (\frac{\lambda_*}{\lambda}-1\right )^2}$ , ${\mu \eqdef \left[\frac{\dvar^T\obs}{\lambda \|\dvar \|^2_2}\right]
 ^{\|\D^T \dvar \|_{\infty}^{-1}}_{-\|\D^T \dvar \|_{\infty}^{-1}}}$ and
  $\atom_* \eqdef \argmax{\atom\in \{ \pm \atom_i \}_{i=1}^K}  \atom^T\obs$.
\end{lem}

When applied with $\dvar = {\obs}$ this screening test is exactly the Dome test~\cite{Xiang2012}. \\

Using these Lemmata at dual points $\dvar_{\it}$ obtained during iterations allows to progressively reduce the radius $\sphrad_{\dvar}$ of the considered 
regions ---sphere or dome--- and thus improves the screening capacity of the screening tests associated to these regions.
The effect of the radius appears clearly in~(\ref{eqn:safe},\ref{eqn:DST3},\ref{eqn:dome1},\ref{eqn:dome2}). 
Note that the choice of a new $\dvar$, for one of the previous screening tests $\ST_{\dvar}$, only acts on $\sphrad_{\dvar}$ 
the radius of the region and not on $\sphcent$ its center.

%
%
%
%
%

\subsubsection{Dynamic screening for the \glasso}\label{sec:dynscr_glasso}

The \lasso problem embodies the assumption that observation $\obs$ may be approximately represented in $\D$ by a 
sparse vector $\pvaropt$. When a particular structure of the data is known, we may additionally assume 
that, besides sparsity, the representation of $\obs$ in $\D$ fits this structure. Inducing the 
structure into $\pvaropt$ is exactly the goal of structured-sparsity
regularizers. Among those we focus on the \glasso regularization because the group-separability of its objective 
function~\eqref{eqn:group-lasso} particularly fits the screening framework.

\paragraph{The \glasso~\cite{Yuan2006}}

The \glasso is a sparse least-squares problem that assumes some group structure in the sparse solution, in the sense 
that there are groups of zero coefficients in the solution $\pvaropt$.
This structure, assumed to be known in advance, is characterized by 
$\G$, a known partition of $\fullset$, and $ \grwght_{g} > 0 $ the weights associated with each group $g 
\in \G$ (\eg $\grwght_g = \sqrt{ \card{g}}$).
Using the group-sparsity inducing regularization $\spreg( \pvar ) \eqdef \sum_{g \in \G} \grwght_g \normII{\grindex{\pvar}{g}}$, the \glasso is defined as:

\begin{align}
&\mathcal{P}^\textit{\glasso}(\lambda, \D, \G, \obs) :\nonumber\\ &\hspace{1cm}\argmin{ \pvar } 
\dfrac{1}{2}\normII{\D \pvar - \obs}^2 + \pen \sum_{g \in \G} \grwght_g \normII{\grindex{\pvar}{g}}. 
\label{eqn:group-lasso}
\end{align}

The proximal operator of the group sparsity regularization is the group soft-thresholding:
\begin{align}
\forall g \in \G, \prox{\glasso}{t}{\pvar_g} \eqdef  \max\left (0,\dfrac{\normII{\pvar_g}-t \grwght_g }{\normII{\pvar_g}}\right ) 
\pvar_g. \label{eqn:gr_softTh}
\end{align}

The dual of the \glasso problem~\eqref{eqn:group-lasso} is (see~\cite{Wang2012}):
\begin{subequations}\label{eqn:dual_glasso}
\begin{align} 
          \dvaropt \triangleq \argmax{\dvar} &\frac{1}{2}\normII{\obs}^2-\frac{\lambda^2}
          {2}\normII{\dvar-\frac{\obs}{\lambda}}^2  \label{eqn:dual_unconstr_glasso}    \\
         \text{ s.t.  } \forall g \in \G, &\dfrac{\normII{\grindex{\D}{g}^T\dvar}}{\grwght_g}\leq 1
 .\label{eqn:dual_constr_glasso}
\end{align}
\end{subequations}

A dual point $\dvar\in \R^{\dimS}$ is said \textit{feasible} for the \glasso if it satisfies constraints~\eqref{eqn:dual_constr_glasso}.


From the convex \emph{optimality conditions}, primal and dual optima are necessarily linked by:
\begin{align}
\obs = \D\pvaropt + \pen \dvaropt,
\tx{ and }
\forall g\in\G
\begin{cases}
\|{\grindex{\D}{g}^T\dvaropt} \|_2 \leq \grwght_g  &\tx{ if  } \grindex{\pvaropt}{g} = \myvec{0},\\
\|{\grindex{\D}{g}^T\dvaropt}\|_2 = \grwght_g &\tx{ if  } \grindex{\pvaropt}{g} \neq \myvec{0}.
\end{cases}\label{eqn:primaldualKKT_glasso}
\end{align}

We now adapt the definition of $\lambda_*$ to the \glasso so that it corresponds to the smallest regularization parameter resulting
into a zero solution of~\eqref{eqn:group-lasso}.
Let us define  
\begin{align}
g_* \eqdef \argmax{g}\dfrac{\normII{\grindex{\D}{g}^T\obs}}{\grwght_g}, \quad 
\lambda_* \eqdef \dfrac{\normII{\grindex{\D}{g_*}^T\obs}}{{\grwght_{g_*}}}.
\label{eqn:gstar}
\end{align}

As for the \lasso, if  $\pen > \lambda_*$ one may screen all the atoms and obtain $\pvaropt = \myvec{0}$. 
Hence for the \glasso setting, we focus on the non-trivial case $\pen \in \left]0,\lambda_*\right]$.

Instances of screening tests for the \glasso are presented in the sequel. 
We extend here the $\safe{}$ ~\cite{ElGhaoui2010} and the $\hsst{}$~\cite{Xiang2011} screening tests to the \glasso.
To our knowledge there are no published results on this extension to the \glasso. 

\paragraph{Screening tests for the \glasso}
As just previously for the \lasso, the quantity $\| {\grindex{\D}{g}^T \dvaropt}\|_2$ in relation~\eqref{eqn:primaldualKKT_glasso}
 is not known except if the problem is solved. Regions $\mathcal{R}$
containing the optimum $\dvaropt$ are considered to use the upper bound $\max_{\dvar \in \mathcal{R}} \|{\grindex{\D}{g}^T \dvar} \|_2$
to identify some inactive groups $g$ thanks to relation~\eqref{eqn:primaldualKKT_glasso}.
Please see the proofs in the appendix for details on the construction of the regions and the solution of the maximization problem.

For any index $i \in \fullset$, we denote by $\gf{i}$ the unique group $g \in \G$ that contains $i$.
The following Lemma extends the SAFE screening test to the \glasso:

\begin{lem}[The Group-SAFE: $\gsafe{}$]
\label{thm:GSAFE}
For any ${ \dvar \in \R^{\dimS} }$, the following function $\ST^{\gsafe{}}_{\dvar}$ is a screening test for $\mathcal{P}^\textit{\glasso}(\lambda, \D, \G, \obs)$: 
\begin{align}
\ST^{\gsafe{}}_{\dvar} : \fullset &\rightarrow \{0,1\}   \nonumber \\
 i & \mapsto  \left \llbracket  \left ( \dfrac{ \grwght_{\gf{i}} }{ \normM{\grindex{\D}{\gf{i}}} } 
 -  \dfrac{\normII{\grindex{\D}{\gf{i}}^T\sphcent}}{\normM{\grindex{\D}{\gf{i}}}} \right )
 >\sphradf{\dvar}\right \rrbracket \label{eqn:gsafe}
\end{align}
where
\begin{align*} 
 \sphcent \eqdef \dfrac{\obs}{\lambda},
 \sphradf{\dvar} \eqdef \normII{\dfrac{\obs}{\lambda} - \mu \dvar}\text{ and } 
 \mu \eqdef \left[\frac{\dvar^T\obs}{\lambda \|\dvar \|^2_2} \right] 
_{-\min\limits_{g \in \mathcal{G}}\frac{\grwght_g}{\normII{\grindex{\D}{g}^T \dvar}}}
^{\min\limits_{g \in \mathcal{G}}\frac{\grwght_g}{\normII{\grindex{\D}{g}^T \dvar}}}.
\end{align*}
\end{lem}

The following Lemma extends the screening tests ST3~\cite{Xiang2011} and DST3 (see Lemma~\ref{thm:HSST}) to the \glasso.

\begin{lem}[The Dynamic Group ST3: $\ghsst{}$]\label{thm:GHSST}
For any ${ \dvar \in \R^{\dimS} }$, the following function $\ST^{\ghsst{}}_{\dvar}$ is a screening test for $\mathcal{P}^\textit{\glasso}(\lambda, \D, \G, \obs)$: 
\begin{align}
\ST^{\ghsst{}}_{\dvar} : \fullset &\rightarrow \{0,1\}   \nonumber \\
 i & \mapsto \left \llbracket \left ( \dfrac{ \grwght_{\gf{i}} }{ \normM{\grindex{\D}{\gf{i}}} } 
 - \dfrac{\normII{\grindex{\D}{\gf{i}}^T\sphcent}}{\normM{\grindex{\D}{\gf{i}}}} \right )
 >\sphradf{\dvar} \right \rrbracket \label{eqn:dgst3}
\end{align}
where
\begin{align*}
 \sphcent & \eqdef \left (\Id - \dfrac{\nvec\nvec^T}{\normII{\nvec}^2}\right )\dfrac{\obs}{\lambda} + \dfrac{\nvec}{\normII{\nvec}^2}	\grwght_{g_*}^2, \\
\nvec & \eqdef \grindex{\D}{g_*}\grindex{\D}{g_*}^T\dfrac{\obs}{\lambda_*}, \\
\sphradf{\dvar} & \eqdef \sqrt{\normII{\dfrac{\obs}{\lambda}- \mu \dvar}^2 - \normII{\dfrac{\obs}{\lambda} - \sphcent }^2} \text{ and } \\
 \mu & \eqdef \left[\frac{\dvar^T\obs}{\lambda \|\dvar \|^2_2} \right] 
_{-\min\limits_{g \in \mathcal{G}} \frac{\grwght_g}{\normII{\grindex{\D}{g}^T \dvar}}}
^{\min\limits_{g \in \mathcal{G}} \frac{\grwght_g}{\normII{\grindex{\D}{g}^T \dvar}}}.
\end{align*}
\end{lem}

In these two Lemmata, the regions $\mathcal{R}$ used to define the screening tests are spheres and the effect 
of the radius $\sphrad_{\dvar}$ on the screening capacity is visible in~(\ref{eqn:gsafe}) and (\ref{eqn:dgst3}).

The proposed screening tests have been given for the \glasso formulation, but  can be readily
extended to the \emph{Overlapping} \glasso~\cite{Jacob2009Group} thanks to the 
replication trick.

\subsection{A turnkey instance}
\label{sec:all_in_one}
As a concrete instance of Algorithm~\ref{alg:GalDynScreenDetail}, we propose to focus on the 
\lasso problem solved by the combined use of ISTA and SAFE.
We compare the static screening with the dynamic screening, through implementations given in 
Algorithms~\ref{alg:static_ista} and~\ref{alg:dynamic_ista}, respectively.
The usual ISTA update appears at lines~\ref{inAlg:ISTA_static1} to~\ref{inAlg:ISTA_static2} in 
Algorithm~\ref{alg:static_ista} and lines~\ref{inAlg:ISTA_dynamic1} to~\ref{inAlg:ISTA_dynamic2} in 
Algorithm~\ref{alg:dynamic_ista}, where the step size $L_{\it}$ is set using the backtracking 
strategy as described in~\cite{Beck2009}. 
The remaining lines of the algorithm, dedicated to the screening process, are described separately in the following paragraphs.

\begin{figure}[tbh]
\begin{center}
\begin{tiny}
\begin{minipage}[t]{.47\textwidth}
\begin{algorithm}[H]
\caption{ISTA + \emph{Static} SAFE Screening}
\label{alg:static_ista}
\begin{algorithmic}[1]
\Require $\dict, \obs, \lambda, {\pvar}_0 \in \R^{\nAt} $
\State  ..........  \emph{Screening} ..........
\State $\scr \gets \left\{  i \in \fullset, \abs{\atom_{i}^T\obs} < \pen - 1 + \frac{\pen}{\lambda_*} \right\}$\label{inAlg:SAFE_static1}
\State $\D_{0} \gets \grindex{\D}{\comp{\scr}}, $\label{inAlg:SAFE_static2}
\State $\it \gets 1$
\While{stopping criterion on ${\pvar}_\it$}
 	\State ... \emph{ISTA update }.....
 	\State $\dvar_{\it} \gets \D_{0} \pvar_{\it-1} -\obs  $
 	\label{inAlg:ISTA_static1}
 	\State $ \pvary_{\it} \gets \D_{0}^T \dvar_{\it}$
 	\State $ \pvar_{\it} \gets \prox{\lasso}{\lambda/L_\it}{\pvar_{\it-1} - \frac{1}{L_{\it}}\pvary_{\it}}$
 	\label{inAlg:ISTA_static2}
	\State $\it \gets \it+1$ 
\EndWhile \\
\Return $\pvar_{\it}$
\end{algorithmic}
\end{algorithm}
\end{minipage}
\begin{minipage}[t]{.47\textwidth}
\begin{algorithm}[H]
\caption{ISTA + \emph{Dynamic} SAFE Screening}
\label{alg:dynamic_ista}
\begin{algorithmic}[1]
\Require $\dict, \obs, \lambda, {\pvar}_0 \in \R^{\nAt}$
\State $\scr_{0} \gets \emptyset$, $\sphrad_0 \gets +\infty, \D_{0} \gets \D$ 
\State $\it \gets 1$
\While{stopping criterion on ${\pvar}_\it$}
 	\State ....\emph{  ISTA update  }.....
 	\State $\dvar_{\it} \gets \D_{\it-1} \bar \pvar_{\it-1} -\obs  $
 	\label{inAlg:ISTA_dynamic1}
 	\State $ \pvary_{\it} \gets \D_{\it-1}^T \dvar_{\it}$
 	\State $  \pvar_{\it } \gets \prox{\lasso}{\lambda/L_\it}{ \bar \pvar_{\it-1} - \frac{1}{L_{\it}}\pvary_{\it}}$
 	\label{inAlg:ISTA_dynamic2}
	\State ............ \emph{Screening} ............
	\State $\mu_{\it} \gets \left[ \frac{\dvar_{\it}^T\obs}{\lambda \|\dvar_{\it} \|^2_2}  \right]_{-\|\pvary_{\it} \|^{-1}_{\infty}}^{\|\pvary_{\it} \|^{-1}_\infty}$ \label{inAlg:SAFE_dynamic1}
%
	\State $\feas_{\it} \gets \mu_\it {\dvar_{\it}}$
	\State $\sphrad_{\it} \gets \normII{\frac{\obs}{\lambda} - \feas_{\it}}$ \label{inAlg:SAFE_dynamic_sphrad}
	\State $\scr_{\it} \gets \left\{  i \in \fullset, 
			\abs{\atom_{i}^T\obs} < \pen (1-\sphrad_{\it} )\right\} \cup \scr_{\it-1}$
			\label{inAlg:SAFE_dynamic_screenset}
	\State $\D_{\it} \gets  \D_{\it-1} \grindex{\Id}{\comp{\scr}_{\it-1},\comp{\scr}_{\it} }$
	 \label{inAlg:SAFE_dynamic2}	
	\State $ \bar \pvar_{\it} \gets \grindex{\Id}{\comp{\scr}_{\it},\comp{\scr}_{\it-1} } \pvar_{\it }$
	\label{inAlg:screenPrimalISTA}
	\State $\it \gets \it+1$ 
\EndWhile\\
\Return $\pvar_{\it}$
\end{algorithmic}
\end{algorithm}
\end{minipage}
\end{tiny}
\end{center} 
\end{figure}

The state-of-the-art static screening shown in 
Algorithm~\ref{alg:static_ista} is the successive use of the SAFE 
screening test ---Lemma~\ref{thm:SAFE} with $\dvar = \obs$ results exactly in
lines~\ref{inAlg:SAFE_static1}-\ref{inAlg:SAFE_static2}--- \emph{prior} to the ISTA algorithm.  
The dictionary is screened once for all, using information from the initially-available data 
$\D^T\obs$ and $\pen$.

The proposed dynamic screening principle is shown in Algorithm~\ref{alg:dynamic_ista}. 
The iteration is here composed of two stages:
a) the ISTA update (lines~\ref{inAlg:ISTA_dynamic1}-\ref{inAlg:ISTA_dynamic2}) which 
is exactly the same as lines~\ref{inAlg:ISTA_static1}-\ref{inAlg:ISTA_static2} of 
Algorithm~\ref{alg:static_ista} except that the dictionary changes along iterations and b)
 the screening step 
(lines~\ref{inAlg:SAFE_dynamic1}-\ref{inAlg:SAFE_dynamic2}) 
which aims at reducing the dictionary size thanks to the information
contained in the current iterates $\dvar_{\it}$ and $\pvary_{\it}$. 
The screening process appears at 
lines~\ref{inAlg:SAFE_dynamic_sphrad}-\ref{inAlg:SAFE_dynamic2} where 
the index sets 
${\scr}_{\it}$ of screened atoms form a non-decreasing inclusion-wise sequence (line~\ref{inAlg:SAFE_dynamic_screenset}).


\subsection{Computational complexity of the dynamic screening}
\label{sec:optimandscreen_comput}

The screening test introduces only a negligible computational overhead because it mainly relies on the matrix-vector 
multiplications already performed in the first-order algorithm update. 
We present now the computational ingredients of the acceleration obtained by dynamic screening.

Algorithm~\ref{alg:GalDynScreenDetail} implements the iterative alternation of the update 
of any first-order algorithm ---\eg those from Table~\ref{tab:optimstep}--- at line~\ref{inAlg:optim}, and a 
screening process at line~\ref{inAlg:ScrTest}-\ref{inAlg:screenPrimal}.
This screening process consists of the pairing of two distinct stages. 
First the set $\scr_{\it}$ of screened atoms is computed at line~\ref{inAlg:ScrTest} using one of the Lemmata~\ref{thm:SAFE} to~\ref{thm:GHSST}.
Analyzing these Lemmata shows that the expensive computation required to 
evaluate the screening test $\ST_{\dvar_{\it}}(i)$ for all $i \in \fullset$ is both due to the products 
$\atom_{i}^T \sphcent$ for all $i \in \fullset$ and to the computation of the scalar $\mu$ which needs the product $\D_{\it}^T\dvar_{\it}$. 
Thus, determining a set of inactive atoms may cost $\mathcal{O}(\nAt 
\dimS)$ per iteration. Fortunately, the computation $\D^T\sphcent$ can be done once for all at 
the beginning of the algorithm. Table~\ref{tab:optimstep} shows that the computation 
$\D_{\it}^T \dvar_{\it}$ is already done by every first-order algorithm. So determining $\scr_{\it}$ 
produces an overhead of $\mathcal{O}(\card{\comp{\scr}_{\it-1}} + \dimS)$ only.
Second the proper screening operations reduce the size of the dictionary and the primal 
variables at lines~\ref{inAlg:screenDict} and~\ref{inAlg:screenPrimal}, with a small computation requirement 
because matrix $\grindex{\Id}{\comp{\scr}_{\it},\comp{\scr}_{\it-1}}$ has 
only $\card{\comp{\scr}_{\it}}$ non-zero elements. So, finally, the computation overhead entailed by the 
embedded screening test has complexity $\mathcal{O}(\card{\comp{\scr}_{\it-1}} + \dimS)$ at iteration 
$\it$ which is negligible compared with the complexity $\mathcal{O}(\card{\comp{\scr}_{\it-1}}  
\dimS)$ for the optimization update $\step(\cdot)$. A detailed complexity analysis is given in Section~\ref{sec:detailed_complexity}.
Finally the total computation cost of the algorithm with dynamic screening may be much smaller than 
the base first-order algorithm. 
This is evaluated experimentally in Section~\ref{sec:simulations}.

%% file: Experiments_single.tex
\section{Experiments}

\label{sec:simulations}
This section is dedicated to experiments made to assess the practical 
relevance of the proposed dynamic screening principle\footnote{The code 
in Python and data are released for reproducible research purposes at 
\url{http://pageperso.lif.univ-mrs.fr/~antoine.bonnefoy}.}. More precisely, 
we aim at providing a rich understanding of its properties beyond what the 
theory can demonstrate. The questions of interest deal with the computational 
performance and may be formulated as follows:

\begin{itemize}
 \item how to measure and evaluate the benefits of dynamic screening?
 \item what is the efficiency of dynamic screening in terms of the overall
 		acceleration compared to the algorithm without screening or with static screening?
 \item to which extent does the computational gain depend on problems, 
 		algorithms, synthetic and real data, screening tests?
\end{itemize}

\subsection{How to evaluate: performance measures} \label{sec:detailed_complexity}
Let us first notice from Theorem~\ref{thm:DynScrCvg} that whatever strategy is used
---no-screening/static screening/dynamic screening---, the algorithms converge 
to the same optimal $\pvaropt$. Consequently, there is no need to evaluate the \textit{quality} 
of the solution and we shall only focus on \textit{computational aspects}.

The main figure of merit that we use is based on an estimation of the number
of floating-point operations (flops) required by the algorithms with no screening
($\flops{N}$), with static screening ($\flops{S}$) and with dynamic screening
($\flops{D}$) for a complete run. We will represent experimental results by
the normalized number of flops $\frac{\flops{S}}{\flops{N}}$ and
$\frac{\flops{D}}{\flops{N}}$ that reflect the acceleration obtained
respectively by the static screening and dynamic screening strategies over the base
algorithm with no screening. Computing such quantities requires to experimentally
record the number of iterations $\it_{f}$ and for each iteration $\it$, the size of the
dictionary $\card{\comp{\scr}_\it}$ and the sparsity of the current iterate $\normO{\pvar_{\it}}$.
They are defined for the \lasso as:
\begin{center}
\begin{footnotesize}
\begin{tabular}{|c|c|} 
\hline $\flops{N}$
 &$ \sum_{\it=1}^{\it_{f}} \left[ (\nAt+\normO{\pvar_{\it}}) \dimS 
 + 4 \nAt+ \dimS \right]$  \\  
$\flops{S}$
& $ \nAt  \dimS+\sum_{\it=1}^{\it_{f}} \left[( \card{\comp{\scr}_0} 
+\normO{\pvar_{\it}})\dimS + 4 \card{\comp{\scr}_0}+ \dimS\right]$  \\  
$\flops{D}$ 
& $\sum_{\it=1}^{\it_{f}} \left[( \card{\comp{\scr}_\it} +\normO{\pvar_{\it}}) 
 \dimS + 6 \card{\comp{\scr}_\it}+ 5\dimS\right]$  \\ 
\hline 
\end{tabular} 
\end{footnotesize}
\end{center}
and for the \glasso as:
\begin{center}
\begin{footnotesize} 
\begin{tabular}{|c|c|} 
\hline $\flops{N}$ 
&$ \sum_{\it=1}^{\it_{f}} \left[(\nAt+\normO{\pvar_{\it}}) \dimS 
+ 4 \nAt+ \dimS + 3\card{\G} \right]$  \\  
$\flops{S} $
& $ \nAt  \dimS+\sum_{\it=1}^{\it_{f}} \left[( \card{\comp{\scr}_0} 
+\normO{\pvar_{\it}})\dimS + 4 \card{\comp{\scr}_0}+ \dimS + 3\card{\G}\right]$  \\  
$\flops{D}$ 
& $\sum_{\it=1}^{\it_{f}} \left[( \card{\comp{\scr}_{\it}} +\normO{\pvar_{\it}})  \dimS 
+ 7 \card{\comp{\scr}_\it}+ 5\dimS + 5 \card{\G}\right]$  \\ 
\hline 
\end{tabular} 
\end{footnotesize}
\end{center}

Indeed, one update of a first-order algorithm at iteration $\it$ requires at least
$2 \card{\comp{\scr}_{\it}} \dimS + \card{\comp{\scr}_{\it}} + \dimS$ to compute the gradient, and
 the proximal operator of the \lasso and \glasso need,  $3 \card{\comp{\scr}_{\it}}$ and $3 \card{\comp{\scr}_{\it}} + 3\card{\G}$ operations, respectively (see Table~\ref{tab:optimstep}, \eqref{eqn:softTh} and~\eqref{eqn:gr_softTh}).
The dynamic screening requires the computation of $\mu$: $2\dimS + \card{\comp{\scr}_{\it}}$ and $2\dimS + 2\card{\comp{\scr}_{\it}} + 2\card{\G}$ for \lasso and \glasso respectively, (see Lemma~\ref{thm:SAFE}-\ref{thm:GHSST}), and $2 \dimS$ for the computation of the 
$\sphrad_{\dvar}$. The screening step is then computed in $\card{\comp{\scr}_{\it}}$ operations.
The static screening approach implies a separated initialization of the screening test which requires $\nAt  \dimS$ operations.

Note that the primal variable which would be sparse during the optimization procedure which reduce the 
number of operation required for $\D_{\it}\pvar_{\it}$ from $\card{\comp{\scr}_{\it}}\dimS$ to 
$\normO{\pvar_{\it}} \dimS$. Note also that here we do not take into account the time required to compute the matrix norm of the sub-dictionaries corresponding to each group: $\norm{\grindex{\D}{g}}, g \in \G$. 
These quantities do not 
depend on the problem $\mathcal{P}(\pen, \spreg, \D, \obs)$ but on the dictionary and the groups 
themselves only, so we consider that they can be computed beforehand for a given dictionary D and a given partition G.

Another option to measure the computational gain consists in actual running times, which we consider as well.
The main advantage of this measure would be that it results from actual performance in seconds instead of estimated 
or asymptotic figure. However, running times depend on the implementation so that it may not be the right measure 
in the current context of algorithm design. 
For each screening strategy (no screening/static/dynamic) we measure the running times $t_{N}/t_{S}/t_{D}$.
The performance are then represented in terms of normalized running times $\frac{t_{D}}{t_{N}}$ and $\frac{t_{S}}{t_{N}}$.
Eventually, one may wonder whether those measures, flops and times, are somehow 
equivalent, which will be checked and discussed in Sections~\ref{sec:exp_lasso} and~\ref{sec:exp_gl}.

\subsection{Data material}
\label{sec:data_mater}

\subsubsection{Synthetic data}
For experiments on synthetic data, we used two types of dictionaries that are widely used in the state-of-the-art of sparse estimation and screening tests. 
The first one is a normalized Gaussian dictionary in which all atoms $\atom_i$ are drawn i.i.d. uniformly on the unit sphere, \eg, by normalizing realizations of $\mathcal{N}(\mathbf{0},\mathbf{Id}_N)$. 
The second one is the so-called Pnoise introduced in~\cite{Xiang2012}, for which all $\atom_i$ are drawn 
i.i.d. as $ \mathbf{e}_1 + 0.1 \kappa \mathbf{g}$ and normalized, where
 $\mathbf{g} \sim \mathcal{N}(\mathbf{0}, \mathbf{Id}_N)$, 
 $ \kappa \sim \mathcal{U}(0,1) $ and  $\mathbf{e}_1 \eqdef [ 1, 0, \ldots, 0 ]^T \in \R^\dimS$ is the first natural basis vector. 
 We set the data dimension to $N\triangleq 2000$ 
 and the number of atoms to $K\triangleq 10000$.
 
In the experiments on the \lasso, observations $\obs$ were drawn i.i.d. from the exact same distribution as the atoms of dictionaries 
described above. In experiments on the \glasso, all groups were built randomly with the same number of atoms in each group.
Observations were generated from a Bernouilli-Gaussian distribution: 
$\card{\G}$ independent draws of a Bernoulli distribution of parameter $p=0.05$ were used to determine for each group if it was active or not.
Then coefficients of active groups are drawn i.i.d. from a standard Gaussian distribution while they are set to zero in inactive groups. 
The observation $\obs$ was generated as the ($l_2$-normalized) sum of $\D \pvar$ and of a Gaussian noise such that the signal-to-noise ratio equals 20dB.

\subsubsection{Audio data} 
For experiments on real data we performed the estimation of the sparse 
representation of audio signals in a redundant Discrete Cosine Transform (DCT) 
dictionary, which is known to be adapted for audio data. Music and speech recordings were taken from the material of the 2008 Signal Separation Evaluation Campaign~\cite{Vincent2009}. We considered 30 observations $\obs$ with length $N=1024$ and sampling rate 16 kHz and the number of atoms is set to {$K\triangleq 10000$}. 

\subsubsection{Image data}
Experiments on the MNIST database~\cite{LeCun1998} have been performed too. 
The database is composed of images of $N \triangleq 28 \times 28 = 784$ pixels representing 
handwritten digits from 0 to 9 and is split into a training set and a testing set.
The dictionary $\D$  is composed of $K \triangleq 10000$ vectorized images from the training set, 
with $1000$ randomly-chosen images for each digit.
Observations were taken randomly in the test set.

\subsection{Solving the \lasso with several algorithms.}
\label{sec:exp_lasso}

We addressed the \lasso problem with four different algorithms from Table~\ref{tab:optimstep}: 
ISTA, FISTA, SpaRSA and Chambolle-Pock. Algorithms stop at iteration $\it$ if either $\it > 200$ 
or the relative variation $\frac{\abs{ F(\pvar_{\it-1}) - F(\pvar_{\it})}}{F(\pvar_{\it})}$ of the objective 
function $F(\pvar)\triangleq \frac{1}{2}\normII{\D \pvar - \obs}^2 + \pen \spreg\left(\pvar\right)$ 
is lower than $10^{-7}$. We used a Pnoise dictionary and three different strategies for 
each algorithm: no screening, static ST3 screening and dynamic ST3 screening. Algorithms were 
run for several values of  $\lambda$ to assess the performance for various sparsity levels.


\begin{minipage}{0.45\textwidth} 
\begin{figure}[H]
\includegraphics[width=1. \textwidth]
{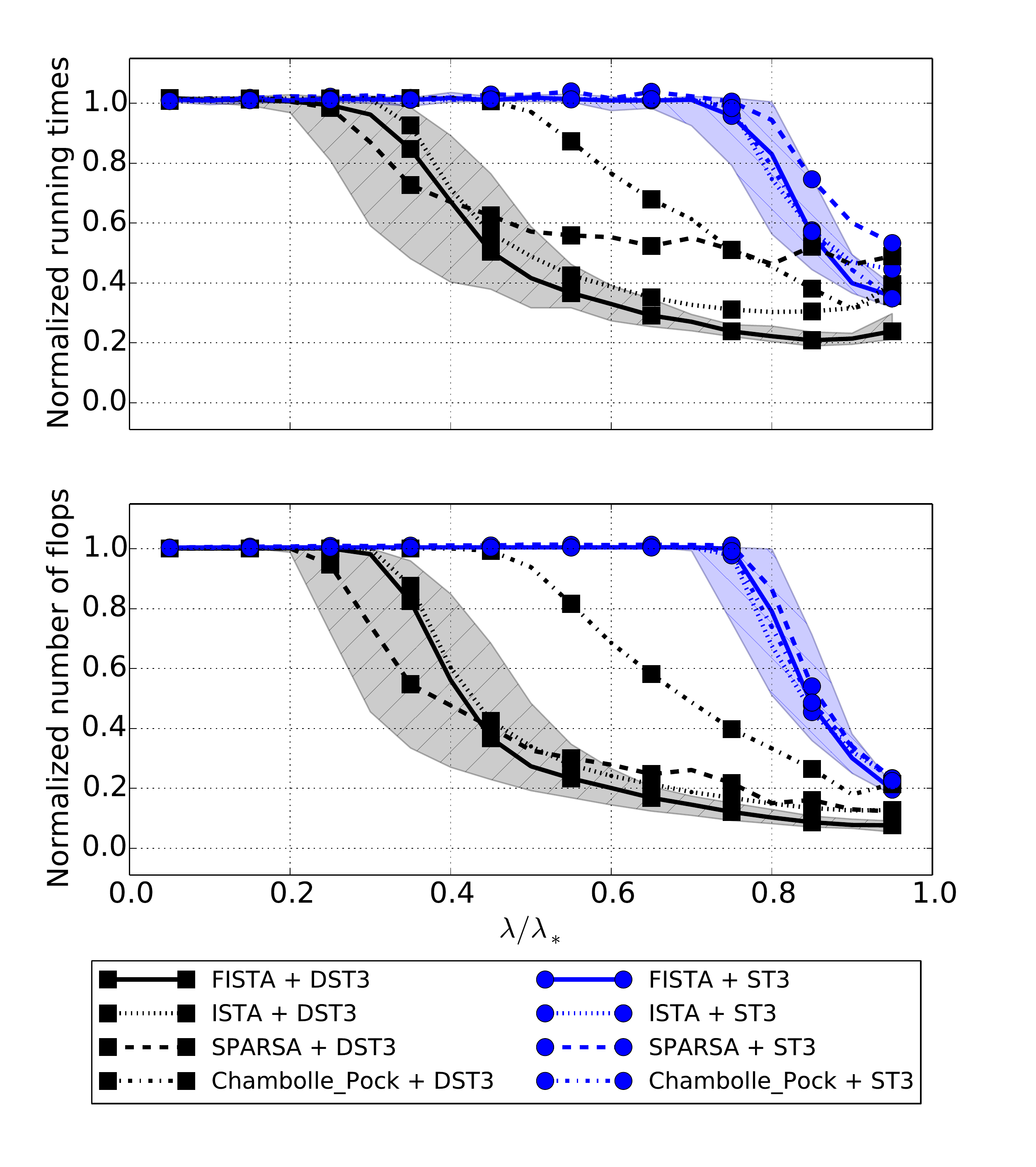} 
\caption{Normalized running times and normalized number of flops for solving the \lasso with a Pnoise dictionary.}
\label{fig:expeRelTime} 
\end{figure}
\end{minipage}
\hfill
\begin{minipage}{0.45\textwidth}
\begin{figure}[H]
\centering 
\includegraphics[width=.92 \textwidth]{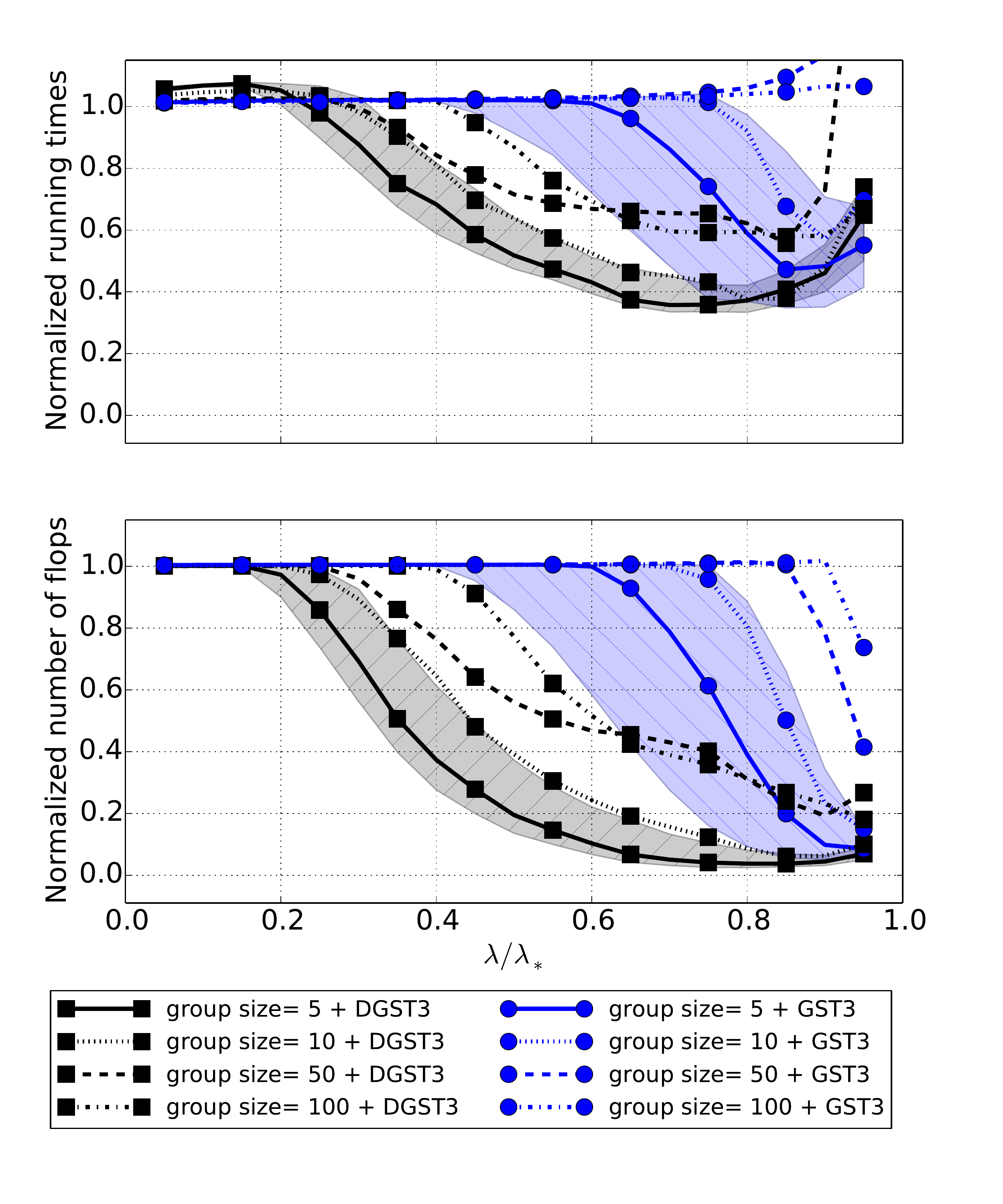} 
\caption{Normalized running times an number of flops for FISTA solving the \glasso with different group sizes (5, 10, 50, 100).}
\label{fig:multigroupsize}
\end{figure} 
\end{minipage}
\vspace{0.5cm}


 Figure~\ref{fig:expeRelTime} shows the normalized running  times and normalized number of flops
 for algorithms with dynamic screening (black squares)
 and for the corresponding algorithms with static screening (blue circle) as
 a function of ${\lambda}/{\lambda_*}$. Lower values account for faster computation. 
The medians among 30 runs are plotted and
the shaded area contains the 25\%-to-75\% percentiles for FISTA, in order to illustrate the typical distribution of the values 
(similar areas are observed for the other algorithms but are not reported for readability).

For all algorithms, the dynamic strategy shows a significant acceleration in a wide range of parameter $\lambda \geq 0.3\lambda_*$.
For $\lambda \geq 0.5\lambda_*$, computational savings reach about $75\%$ of the running time and $80\%$ of the number of flops. 
The static strategy is efficient in a much reduced range $\lambda \geq 0.8\lambda_*$, with lower computational gains.
Among all tested algorithms, FISTA has the largest ability to be accelerated, 
which is really interesting as it is also known to be very efficient in terms of 
convergence rate. 
 Note that due to the normalization of running times and flops, 
 Figure~\ref{fig:expeRelTime} cannot be used 
to draw any conclusion on which of ISTA, FISTA, SpaRSA or Chambolle-Pock 
is the fastest algorithm. Finally, one may observe that the running time and
flops measures have similar trends, supporting the idea that only one of them 
may be used to assess computational performance in a fair way.


\subsection{Solving the \glasso for various group sizes}
\label{sec:exp_gl}

We addressed the \glasso with FISTA using Pnoise data and dictionary as described in~\ref{sec:data_mater} 
with several group sizes.
In Figure~\ref{fig:multigroupsize}, the median normalized running times 
and number of flops over 30 runs are plotted, the shaded area representing 
the 25\%-to-75\% percentiles when the group size is 5.


The computational gains obtained by the dynamic screening strategy are 
of the same order than for the \lasso, with large savings in a wide range 
$\lambda \geq 0.3\lambda_*$. One may anticipate that when groups 
grow larger it is more difficult to locate some inactive groups, 
Figure~\ref{fig:multigroupsize} 
confirms this intuition: screening tests become less and less efficient 
to locate inactive groups and consequently the acceleration is not as 
efficient as in the \lasso problem.
As for the \lasso, running times and flops have similar trends, even 
if we observe a larger discrepancy. The discrepancy is due to implementation details. 
For instance the many loops on groups required for the computation of the screening 
tests are hard to handle efficiently in python.

\subsection{Comparing screening tests}

From the previous experiments, we retained 
FISTA to solve the \lasso on synthetic data with a Pnoise dictionary and a 
Gaussian dictionary, on real audio data, and on images. Results are reported 
in Figure~\ref{fig:multiscrtests} for all the proposed screening tests.

\begin{figure}[tbh]
\centering
\includegraphics[width=0.79\textwidth]
{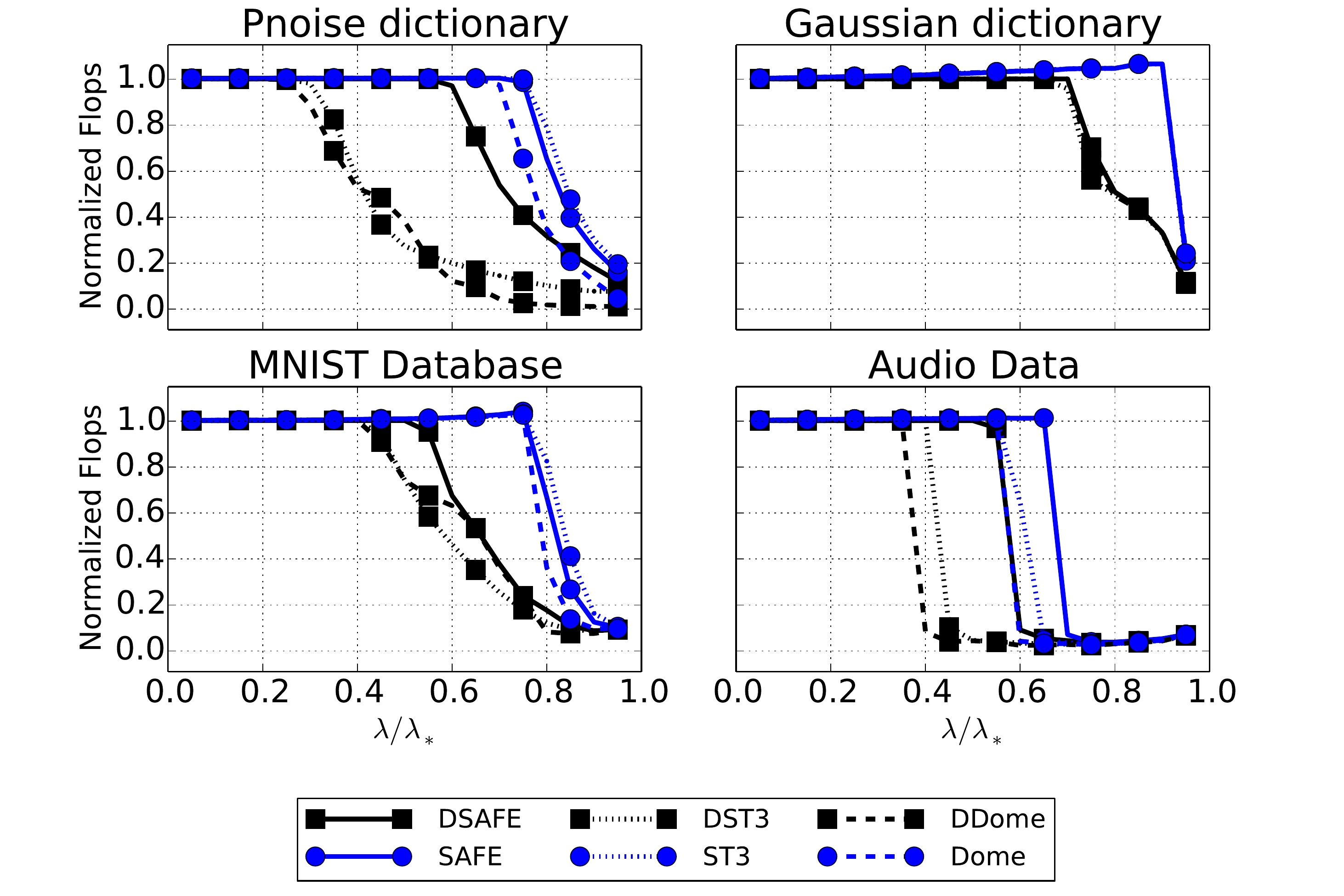} 
\caption{Computational gain of screening strategies for various data and screening tests, on the \lasso solved by FISTA.}
\label{fig:multiscrtests}
\end{figure}

For all kinds of data and all screening tests, dynamic screening again provides a large acceleration
on a wide range of $\lambda$ values and improves the static screening strategy. In the case of
 the Pnoise dictionary and of audio data, the ST3 and Dome tests bring in important improvement 
over the SAFE test, in the static and dynamic strategies. Indeed, $\lambda_*$ is close to 1 in these
cases so that the radius $\sphrad_{\dvar}$ in~\eqref{eqn:safe} for SAFE is much larger than 
in~\eqref{thm:HSST} for ST3 and in~\eqref{thm:HDST} for Dome, which degrades the screening 
efficiency of SAFE. This difference is even more visible when the dynamic screening strategy is used.
As a counterpart, the Gaussian dictionary have small correlation between atoms. 
In this dictionary, ST3 and Dome do not improve the performance of SAFE, but the dynamic 
strategy allows a higher acceleration ratio and for a larger range of parameter $\lambda$.

%
%
%
%
%


%% file: discussion.tex
\section{Discussion}\label{sec:discussion}

We have proposed the dynamic screening principle and shown that this principle is relevant both theoretically and 
practically. When first-order algorithms are used,
dynamic screening induces stronger acceleration on the \lasso and \glasso solvings than static 
screening, and in a wider range of $\lambda$.

The convergence theorem (Theorem~\ref{thm:DynScrCvg}) makes very few assumptions on the iterative algorithm,
meaning that dynamic screening principle can be applied to many algorithms
 ---\eg second order algorithms.
Conversely, dynamic screening tests may produce different iterates than those of the base algorithm on which it is applied and hence 
may modify the convergence rate.
Can we ensure that the dynamic screening 
preserves the convergence rate of any first-order algorithm? Answering this question 
would definitely anchor dynamic screening in a theoretical context.
 
We presented here algorithms designed to compute the \lasso problem for a given $\pen$.
Departing from that, one might be willing
to compute the whole regularization path
as done by the \lars algorithm~\cite{Efron04lars}.
Thoroughly studying how screening might be combined with \lars is another
exciting subject that we plan to work on in a near future.

In a recent work~\cite{Wang13} Wang et. \textit{al} introduce a 
way to adapt the static dome test in a continuation strategy. 
This work relies on exact solutions of successive computation for higher $\pen$ parameters.
Iterative optimization algorithms do not give exact 
solutions hence examining how the dome test can be adapted dynamically in an iterative 
optimization procedure might be of great interest and lead to new approaches.

Given the nice theoretical and practical behavior of Orthogonal Matching
Pursuit~\cite{MallatZ93Matching,Tropp2004},
investigating how it can be paired with dynamic screening is a
pressing and exciting matter but poses the problem of
dealing with the non-convex $\ell_0$ regularization which prevents from 
using the theory and toolbox of convex optimality.

Lastly, as in~\cite{ElGhaoui2010}, we are curious to see how dynamic 
screening may show up when
other than an $\ell_2$ fit-to-data is studied: for example, this situation
naturally occurs when classification-based losses are considered. As
sparsity is often a desired feature for both efficiency (in the
prediction phase) and generalization purposes, being able to
work out well-founded results allowing dynamic screening is of the
utmost importance.

%% file: Appendix_single.tex
\section*{Screening tests for the \glasso}
\label{sec:screen_analyse}
This section is dedicated to the proofs of the screening tests given in the 
Lemmata~\ref{thm:SAFE} to \ref{thm:GHSST}.

\begin{proof}[Proof of Lemmata 2 and 3]
Since the \lasso is a particular case of the \glasso, \ie groups of size one with $\grwght_g = 1$ for all $g \in \G$, 
Lemmata 2 and 3 are direct corollaries of Lemmata 5 and 6.
\end{proof}

\subsubsection*{Base concept}

Extending what has been proposed in~\cite{ElGhaoui2010, Xiang2011}
 we construct screening tests for \glasso using optimality conditions of the
 \glasso~\eqref{eqn:primaldualKKT_glasso}
jointly with the dual problem~\eqref{eqn:dual_glasso}.
These screening tests may locate inactive groups in $\G$.
According to \eqref{eqn:primaldualKKT_glasso}, groups $g$ such that 
${\|{\grindex{\D}{g}^T\dvaropt} \|_2 <\grwght_g}$ correspond to inactive groups
 which can be removed from $\D$.
The optimum $\dvaropt$ is not known but we can construct a region 
$\mathcal{R}\subset \R^N$ that contains $\dvaropt$ so that $\max_{\dvar \in \mathcal{R}} \| \grindex{\D}{g}^T \dvar \|_2$
gives a sufficient condition to screen groups: 
 \begin{align}
 \max_{\dvar \in \mathcal{R}}\normII{\grindex{\D}{g}^T \dvar} < \grwght_{g}
\label{eqn:gal_R_test_glasso} \Rightarrow \normII{\grindex{\D}{g}^T \dvaropt} < \grwght_{g}
\Rightarrow \grindex{\pvaropt}{g} = \myvec{0}
\end{align}



%


There is no general closed-form solution of the maximization problem
in~\eqref{eqn:gal_R_test_glasso} that would apply for arbitrary regions $\mathcal{R}$, 
moreover the quadratic nature of the maximization prevents closed-form 
solutions even for some simple regions $\mathcal{R}$.
We now present the instance of this concept when $\mathcal{R}$ is a sphere.
The sphere centered on $\sphcent$ with radius $\sphrad$ is denoted by $\sph_{\sphcent,\sphrad}$.

\subsubsection*{Sphere tests} 
Consider a sphere $\sph_{\sphcent,\sphrad}$ that contains the dual optimum $\opt\dvar$, 
the screening test associated with this sphere requires to solve $\max_ {\dvar \in \sph_{\sphcent,\sphrad}} \|{\grindex{\D}{g}^T\dvar} \|_2$ for each group $g$, which has no closed-form solution. 
Thus we use the triangle inequality to obtain a closed-form upper-bound on the solution: 
Lemma~\ref{thm:sphere_test_glasso} provides the corresponding sphere-test.

\begin{lem}[Sphere Test for~\glasso]
 \label{thm:sphere_test_glasso}
If $\sphrad\geq 0$ and $\sphcent \in\R^N$ are such that $\dvaropt \in \sph_{\sphcent,\sphrad}$, then the 
following function $\ST^\gsphere{}$ is a screening test for $\mathcal{P}_\glasso(\pen,\D,\G,\obs)$: 

\begin{align}
 \ST^{\gsphere} : \fullset &\rightarrow \{0,1\}   \nonumber \\
 i & \mapsto \left  \llbracket \left ( \dfrac{\grwght_{\gf{i}}}{\normM{\grindex{\D}{\gf{i}}^{ }}}
 -\dfrac{\normII{\grindex{\D}{\gf{i}}^T\sphcent}}{\normM{\grindex{\D}{\gf{i}}^{ }}} \right ) >\sphrad \right  \rrbracket
 \label{eqn:Gsphere}
\end{align}
\end{lem}
\begin{proof}
Let $i \in \fullset$ such that $\ST^{\gsphere}\left(i\right)=1$ and $\sphrad\geq 0$ and $\sphcent \in\R^N$
 are such that $\dvaropt \in \sph_{\sphcent,\sphrad}$.
We use the triangle inequality to upper bound  $\max_{\dvar \in \sph_{\sphcent,\sphrad}}\| \grindex{\D}{g}^T\dvar \|_2$:

\begin{align*}
\max_{\dvar \in \sph_{\sphcent,\sphrad}} \normII{\grindex{\D}{\gf{i}}^T\dvar} &\leq \normII{\grindex{\D}{\gf{i}}^T\sphcent} + \max_{\dvar \in \sph_{\sphcent,\sphrad}} \normII{\grindex{\D}{\gf{i}}^T(\sphcent-\dvar)}\\
&\leq \normII{\grindex{\D}{\gf{i}}^T\sphcent} + \sphrad\normM{\grindex{\D}{\gf{i}}} 
\end{align*}
Which, as $\ST^{\gsphere}(i) = 1$ gives $\normII{\grindex{\D}{\gf{i}}^T \dvar} <\grwght_{\gf{i}} $.
Then using~\eqref{eqn:gal_R_test_glasso} we have that $\grindex{\pvaropt}{\gf{i}} = \myvec{0}$ and $\coef{\pvaropt}{i}=0$.
\end{proof}

\subsubsection*{Dynamic Construction of feasible point using Dual Scaling for the \glasso}


Before giving the proof of Lemmata~\ref{thm:GSAFE} and \ref{thm:GHSST}, we need to 
introduce the dual-scaling strategy that computes from any dual point $\dvar$, 
a feasible dual point that 
satisfies~\eqref{eqn:dual_constr_glasso} by definition 
and aim at being close to ${\obs}/{\lambda}$ to obtain an efficient screening test.
Proposed by El Ghaoui in~\cite{ElGhaoui2010} for the \lasso this method applied to the \glasso is given in the following Lemma:

\begin{lem}\label{thm:dualscaling_glasso}
Among all feasible scaled versions of $\dvar$,
the closest to ${\obs}/{\lambda}$ is $\feas = \mu \dvar$ where:
\begin{align}
\mu & \eqdef \left[\frac{\dvar^T\obs}{\lambda \|\dvar \|^2_2} \right] _{-s_{\min}}^{s_{\min}}
\text{ with } \quad s_{\min} \eqdef \min_{g \in \mathcal{G}}\dfrac{\grwght_g} {\normII{\grindex{\D}{\gf{i}}^T \dvar}} . \label{eqn:dualscaling_glasso}
\end{align}
\end{lem}

\begin{proof}
The dual-scaling problem for the \glasso is: 
\begin{align}
\mu \eqdef \argmin{s \in \R} \normII{s \dvar - \dfrac{\obs}{\pen}} \text{  s.t.  } \forall g \in \G, \normII{\grindex{\D}{g}^T s \dvar} < \grwght_{g}
\label{eqn:dualscal_prob_glasso}
\end{align}
The solution of \eqref{eqn:dualscal_prob_glasso} is the projection onto the feasible segment $[-s_{\min}, s_{\min}] \subset \R$
of the solution of  $\argmin{s \in \R} \normII{s \dvar - \dfrac{\obs}{\pen}}$.
This solution is given in \eqref{eqn:dualscaling_glasso}.
\end{proof}

%
%
We now prove the SAFE test for \glasso using Lemmata~\ref{thm:sphere_test_glasso} and~\ref{thm:dualscaling_glasso} following the same arguments as in \cite{ElGhaoui2010}.

\begin{proof}[Proof of Lemma~\ref{thm:GSAFE}]
Let $\dvar \in \R^\dimS$ and $\feas$ be its feasible scaled version obtained by dual-scaling (Lemma \ref{thm:dualscaling_glasso}).
Since ${\obs} / {\pen}$ attains the minimum of the unconstrained objective~\eqref{eqn:dual_unconstr_glasso} of
 the dual problem~\eqref{eqn:dual_glasso},  and since $\feas$ complies with  
all the constraints~\eqref{eqn:dual_constr_glasso}, the distance between the optimum $\dvaropt$ and ${\obs} / {\pen}$
 is upper bounded by $\normII{\dfrac{\obs}{\pen} - \feas}$, \ie, $\dvaropt\in\sph_{\frac{\obs}{\pen},\normII{\frac{\obs}{\pen} - \feas}}$.
And Lemma~\ref{thm:sphere_test_glasso} concludes the proof.
\end{proof}


\begin{proof}[Proof of Lemma~\ref{thm:GHSST}]
This proof is illustrated graphically in Figure~\ref{fig:ST3_glasso} in 2D.
We first construct geometrically objects that are involved in the proof. Recalling that
\begin{align}
g_* \eqdef \argmax{g}\dfrac{\normII{\grindex{\D}{g}^T\obs}}{\grwght_g} \text{ and } 
\lambda_* \eqdef \dfrac{\normII{\grindex{\D}{g_*}^T\obs}}{{\grwght_{g_*}}},
\end{align}
we define the set of dual point complying with the constraint associated with group $g_*$ as
\begin{align*}
\feasset \eqdef \left \{ \dvar \in \R^{\dimS}, \normII{ \grindex{\D}{g_*} \dvar } \leq \grwght_{g_*}  \right \}
\end{align*}
This set  $\feasset$ is the set contained in the ellipsoid: $${\ellips_* \eqdef \left \{ \dvar \in \R^{\dimS}, \normII{ \grindex{\D}{g_*} \dvar }^2 = \grwght_{g_*}^2  \right \} }.$$
Point $ {\obs} / {\lambda_*}$ is on the ellipsoid: we define $\nvec$ as a normal vector to the ellipsoid $\ellips_*$ at this point. 
Such a vector is built from the gradient of  ${ f\left(\dvar\right) \eqdef \dfrac{1}{2}  \normII{\grindex{\D}{g_*}^T \dvar}^2}$ at $\obs / \lambda_*$:
\begin{align}
\nvec \eqdef \nabla f \left(\dfrac{\obs}{\lambda_*}\right) = \grindex{\D}{g_*}\grindex{\D}{g_*}^T\dfrac{\obs}{\lambda_*}
\end{align}
We denote by $\mathcal{A}$ the half-space defined by  the hyperplane tangent to the ellipsoid at 
$ {\obs}/{\lambda_*}$ that contains $\feasset$:
$
{\mathcal{A} \eqdef \Big\{  \dvar \in \R^\dimS , {\dvar^T\nvec} \leq  \grwght_{g_*}^2  \Big\}.}
$

By construction $\feasset \subset \mathcal{A}$. 
We now construct the new sphere $\sph^\ghsst{}$ containing $\dvaropt$.
Let $\dvar \in \R^{\dimS}$ and $\feas$ the feasible scaled version of $\dvar$ obtained by dual-scaling.
We know $\dvaropt \in \feasset$ (since $\dvaropt$ is feasible) and $\dvaropt 
\in  \sph_{\frac{\obs}{\pen},\normII{\frac{\obs}{\pen} - \feas}}$ (see proof of 
Lemma~\ref{thm:GSAFE}), so we have:
$
\dvaropt \in \feasset \cap \sph_{\frac{\obs}{\pen},\normII{\frac{\obs}{\pen} - \feas}} 
\subset  \mathcal{A}  \cap \sph_{\frac{\obs}{\pen},\normII{\frac{\obs}{\pen} - \feas}} 
$.

\begin{figure}
\center
\includegraphics[width=0.65\textwidth]{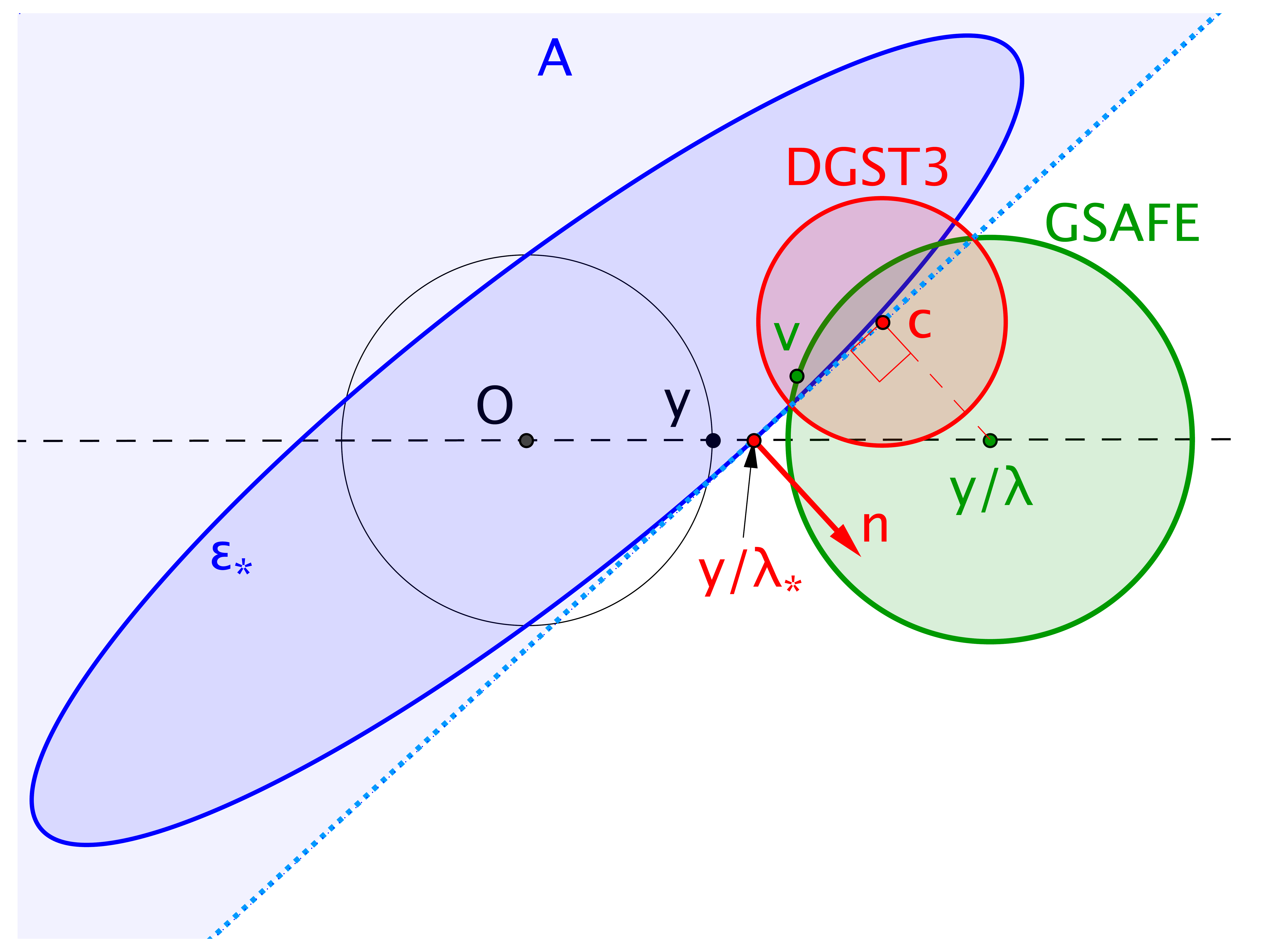} 
\caption{Geometrical illustration of regions associated to the screening tests for the \glasso}
\label{fig:ST3_glasso}
\end{figure}

Then the new sphere $\sph^\ghsst{}$ is defined as the bounding sphere of 
$\mathcal{A}  \cap \sph_{\frac{\obs}{\pen},\normII{\frac{\obs}{\pen} - \feas}} $,
 its center is the projection of $\dfrac{\obs}{\lambda}$ on $\mathcal{H}$ which is given by:
\begin{align*}
\sphcent
&=\dfrac{\obs}{\lambda} - \left( \dfrac{\nvec^T\obs }{\lambda}
- \grwght_{g_*}^2\right)\dfrac{\myvec{n}}{\normII{\nvec}^2}
\end{align*}
and its radius is given by the Pythagoras theorem:
\begin{align*}
\sphrad = \sqrt{\normII{\dfrac{\obs}{\lambda} - \sphcent }^2-\normII{\dfrac{\obs}{\lambda}- \feas }^2}
\end{align*}

We now formally check that $\dvaropt \in \sph_{\sphcent, \sphrad}$ by showing that 
${\mathcal{A} \cap \sph_{\frac{\obs}{\pen},\normII{\frac{\obs}{\pen} - \feas}} \subset \sph^\ghsst{}}$.
Let $\dvar \in \mathcal{A} \cap \sph_{\frac{\obs}{\pen},\normII{\frac{\obs}{\pen} - \feas}}$, then:
\begin{align*}
\normII{ \dfrac{\obs}{\pen} - \feas }^2 
&
\geq \normII{ \dvar - \dfrac{\obs}{\pen} }^2 \\
& = \normII{ \dvar - \dfrac{\obs}{\pen} + \left (\frac{ \nvec^T \obs}{\pen} 
	-  \grwght_{g_*}^2  - \frac{ \nvec^T \obs}{\pen} +  \grwght_{g_*}^2 \right ) \dfrac{\nvec}{\normII{\nvec}^2}}^2 \\
& = \normII{ \dvar - \sphcent }^2 + \normII{ \left (\dfrac{\nvec ^T \obs}{\pen} -  \grwght_{g_*}^2 \right )\dfrac{\nvec}{\normII{\nvec}^2} } ^2 
-2\dfrac{\grwght_{g_*}^2}{\normII{\nvec}^2} \left (\frac{\lambda_*}{\pen} -  1 \right )  
	\left ( \dvar^T \nvec -  \dfrac{\lambda_*}{\pen} \grwght_{g_*}^2 + \left ( \frac{\lambda_*}{\pen} -1 \right ) \grwght_{g_*}^2 \right ).
\end{align*}	
Where the last equality is obtained using the definition of $\nvec$:
\begin{align*}
\dfrac{\nvec ^T \obs}{\pen} -  \grwght_{g_*}^2 =   \frac{\lambda_*}{\pen} \dfrac{\obs^T \nvec}{\lambda_*} - \grwght_{g_*}^2   =  \left ( \frac{\lambda_*}{\pen} -1 \right ) \grwght_{g_*}^2 
\end{align*}
Then, using the definition of $\sphcent$, we have:
\begin{align*}
\normII{ \dfrac{\obs}{\pen} - \feas }^2 &\geq \normII{ \dvar \! - \! \sphcent }^2 + 
	\normII{\dfrac{\obs}{\pen } - \sphcent}^2 
+ 2 \dfrac{\grwght_{g_*}^2}{\normII{\nvec}^2} \! \left ( \! \frac{\lambda_*}{\pen} \! -  \! 1 \! \right )  \! \left ( \grwght_{g_*}^2 \! - \dvar^T \nvec \right )
\!.
\end{align*}
As $\dvar$ is contained in $ \mathcal{A}$ we have $0 \leq \grwght_{g_*}^2 - \dvar^T \nvec$ and:
\begin{align*}
\normII{ \dfrac{\obs}{\pen} - \feas }^2 &\geq  \normII{ \dvar - \sphcent }^2
+\normII{\dfrac{\obs}{\pen } - \sphcent}^2
\end{align*}

We finally obtain the radius:
\begin{align*}
\normII{ \dvar - \sphcent }^2 & \leq  { \normII{ \dfrac{\obs}{\pen} - \feas }^2 -\normII{\dfrac{\obs}{\pen} - \sphcent}^2} = \sphrad^2
.
\end{align*}
Then $\mathcal{A} \cap \sph_{\frac{\obs}{\pen},\normII{\frac{\obs}{\pen} - \feas}} 
\subset \sph^\ghsst{}$ and $\dvaropt \in \sph^{\ghsst{}}$. Lemma~\ref{thm:sphere_test_glasso}
 concludes the proof of Lemma~\ref{thm:GHSST}.
\end{proof}